\documentclass{article}

\usepackage{amsmath,amsthm,amssymb,amsfonts}
\usepackage{stmaryrd}
\usepackage{algorithm}
\usepackage{algorithmic}
\usepackage{graphicx}
\usepackage{titlesec}
\usepackage{hyperref}
\usepackage{xspace}
\usepackage{float}
\usepackage{olo}
\usepackage{caption}
\usepackage{epstopdf}
\usepackage{fullpage, prettyref}
\allowdisplaybreaks

\newcommand{\torrent}{\textsc{Torrent}\xspace}
\newcommand{\myalgo}{\textsc{CRR}\xspace}
\newcommand{\aoard}{\textsc{CRTSE}\xspace}			

\newcommand{\bt}{\vb^t}
\newcommand{\btn}{\vb^{t+1}}
\newcommand{\btp}{\vb^{t-1}}
\newcommand{\bo}{\vb^\ast}
\newcommand{\wt}{\vw^t}
\newcommand{\wo}{\vw^\ast}
\newcommand{\vlt}{\vlambda^t}
\newcommand{\vln}{\vlambda^{t+1}}
\renewcommand{\ko}{k^\ast}
\newcommand{\yo}{\vy^\ast} 					
\newcommand{\uo}{\vu^\ast} 					
\newcommand{\eo}{\ve^\ast} 					

\newcommand{\veps}{\vepsilon}
\newcommand{\vla}{\vlambda}
\newcommand{\<}{\leftarrow}

\renewcommand{\hat}[1]{\widehat{{#1}}}

\renewcommand{\bar}[1]{\overline{{#1}}}

\newcommand{\It}{I^t}
\newcommand{\HT}{\text{HT}}
\newcommand{\Itn}{I^{t+1}}
\newcommand{\mdt}{{\text{MD}^t}}
\newcommand{\fat}{{\text{FA}^t}}
\newcommand{\cit}{{\text{CI}^t}}
\newcommand{\mdn}{{\text{MD}^{t+1}}}
\newcommand{\fan}{{\text{FA}^{t+1}}}
\newcommand{\cin}{{\text{CI}^{t+1}}}
\newcommand{\erfc}{\text{erfc}}
\newcommand{\ARD}{\text{AR}\br{d}}			
\newcommand{\VAR}{\text{VAR}\br{1}}			

\makeatletter
\newcommand{\newreptheorem}[2]{\newtheorem*{rep@#1}{\rep@title} 
\newenvironment{rep#1}[1]{\def\rep@title{#2 \ref*{##1}}\begin{rep@#1}}{\end{rep@#1}}
}
\makeatother

\newreptheorem{lemma}{Lemma}
\newreptheorem{theorem}{Theorem}
\newreptheorem{claim}{Claim}

\title{Efficient and Consistent Robust Time Series Analysis}
\author{Kush Bhatia$^{*}$ \and Prateek Jain$^{*}$ \and Parameswaran Kamalaruban$^{\#}$ \and Purushottam Kar$^{\dagger}$\\ $^{*}$Microsoft Research, Bangalore, India\\\texttt{\{t-kushb, prajain\}@microsoft.com}\\$^{\#}$Australian National University, Canberra, Australia\\
\texttt{kamalaruban.parameswaran@nicta.com.au}\\$^{\dagger}$Indian Institute of Technology Kanpur, India\\
\texttt{purushot@cse.iitk.ac.in}}
\date{}

\begin{document}

\maketitle

\begin{abstract}
We study the problem of robust time series analysis under the standard auto-regressive (AR) time series model in the presence of arbitrary outliers. We devise an efficient hard thresholding based algorithm which can obtain a {\em consistent} estimate of the optimal AR model despite a large fraction of the time series points being corrupted. Our algorithm alternately estimates the corrupted set of points and the model parameters, and is inspired by recent advances in robust regression and hard-thresholding methods. However, a direct application of existing techniques is hindered by a critical difference in the time-series domain: each point is correlated with {\em all} previous points rendering existing tools inapplicable directly. We show how to overcome this hurdle using novel proof techniques. Using our techniques, we are also able to provide the {\em first} efficient and provably consistent estimator for the robust regression problem where a standard linear observation model with white additive noise is corrupted arbitrarily. We illustrate our methods on synthetic datasets and show that our methods indeed are able to consistently recover the optimal parameters despite a large fraction of points being corrupted.
\end{abstract}


\section{Introduction}
Several real world prediction problems, for instance, the temperature of a city, stock prices, traffic patterns, the GPS location of a car etc are naturally modeled as time series. One of the most popular and simple model for time series is the auto-regressive (AR ($d$)) model which models a given observation as a sample from a distribution with mean given by a fixed linear combination of previous $d$ time series values. That is, $x_{t}=\sum_{i=1}^d w^*_ix_{t-i}+\epsilon_i$ where $\epsilon_i$ is unbiased noise.

Unfortunately, in real life scenarios, time series tend to have several outliers. For example, traffic patterns may get disrupted due to accidents and stock prices may get affected by unforseen political or social influences. The estimation of model parameters in the presence of such outliers is a classical problem in time-series literature and is given a detailed treatment in several texts \cite{MaronnaMY2006, Martin1979}. 

Existing time-series texts define two major outlier models: a) innovative outliers, b) additive outliers. In innovative outliers, corrupted values become a part of the time series and influence future iterates i.e. if $x_t$ is corrupted and we observe $\tilde{x}_t=x_t+b_t$ then subsequent values $x_{t'}$  ($t'>t$) are obtained by using $\tilde{x}_t$ rather than $x_t$. In the additive outlier model, on the other hand, although the observation of $\tilde x_t$ is corrupted, the time series itself continues using the clean value $x_t$. Conventional wisdom in time series literature considers innovative outliers to be ``good'' and helpful in spurring a shift in the time series \cite{MaronnaMY2006}. Additive outliers, on the other hand, are considered more challenging due to this latent behaviour in the model and can cause standard estimators for the AR model to diverge.

Due to importance of the problem, several estimators have been proposed for the AR model under corruption, e.g. the generalized M-estimator by \cite{MartinZ1978}. 
However, most existing estimators are computationally intractable (operate in exponential time) and do not offer non-asymptotic guarantees.

Our goal in this work is to devise an {\em efficient} and {\em consistent} estimator for the Robust Time Series Estimation (RTSE) problem in the AR($d$) model with non-asymptotic convergence guarantees in the presence of a large number of outliers. 
To this end, we cast the model estimation problem as a sparse estimation problem and use techniques from the sparse regression literature \cite{JainTK2014} to devise our hard-thresholding based algorithm. At a high level, our algorithm locates the corrupted indices by using a projected gradient method where the projection is onto the set of sparse vectors.

However, analyzing this technique proves especially challenging. While hard threshodling methods have been extensively studied for sparse linear regression \cite{JainTK2014, BlumensathD2009, Zhang11}, similar techniques do not apply directly to our problem because of two key challenges: a) in the time series domain, data points $x_t$'s are dependent on each other while sparse linear regression techniques typically assume independence of the data points, and b) even for robust linear regression (where each row of data matrix is assumed to be independent), existing analyses \cite{BhatiaJK2015} are unable to guarantee consistent estimates. 

Using a novel two-stage proof technique, we show that our method provides a consistent estimator for the true model $\wo$ so long as the number of outliers $k$ satisfies $k=O(\frac{n}{d \log n})$, where $n$ is the total number of points in the time series
and $d$ is the order of the model. Whenever $k$ satisfies the above assumption, our method in time $\tilde{O}(nd)$ outputs an estimate $\hat\vw$ s.t. $\|\hat\vw-\wo\|_2\leq f(n)$ where $f\rightarrow 0$ as $n\rightarrow \infty$. We direct the reader to Theorem~\ref{thm:crtse-final} for precise rates. 

In fact, using our techniques, we are also able to give a {\em consistent} estimator for the robust least squares regression (RLSR) problem \cite{WrightM10,NguyenT13,BhatiaJK2015} even when a constant fraction of the responses are corrupted. Here again, our algorithm runs in time $\tilde{O}(nd)$, where $d$ is the dimensionality of the data. To the best of our knowledge, our method is the {\em first} efficient and consistent estimator for the RLSR problem in the challenging setting where a {\em constant} fraction of the responses can be corrupted. 

We then study our methods empirically for both the robust time series analysis, as well as the standard robust regression problems. Our methods demonstrate consistency for both problem settings. Moreover, our results for robust time series show that the ordinary least squares estimate, that ignores outliers, provides very poor estimators and hence, is significantly less accurate. In contrast, our proposed method and a few variants of it indeed recover the underlying AR($d$) model accurately. 

{\bf Paper Organization}: Section~\ref{sec:rreg} considers the ``warm-up'' problem of robust regression and presents our algorithm and theoretical guarantees. We then, introduce the robust time series problem and our algorithm and analysis in Section~\ref{sec:ts}. Section~\ref{sec:exps} presents simulations on synthetic datasets. 

\section{Related Works}
{\bf Time Series}: Analysing time series with corruptions is a classical and widely studied problem in statistics literature. In an early work, \cite{MartinZ1978} proposed a generalized M-estimator for the RTSE problem in the additive outlier (AO) model with a positive breakdown point. \cite{MaronnaMY2006} detail a robust variant of the Durbin-Levinson algorithm for RTSE and demonstrate the efficacy of the model empirically. \cite{StockingerD1987} provide an analysis of M-estimators for RTSE with innovative outliers (IO), but show that the standard M-estimator has a break down point of {\em zero} in the presence of AO. This shows that standard M-estimators cannot handle even a non-zero fraction of corruptions. Recently, \cite{CrouxJ2008} proposed a method based on Least Trimmed Squares, which is closely related to our method, and used Monte Carlo simulations to validate the effectiveness of their method. \cite{MulerPY2009} present a method based on robust filters in the more powerful ARMA model. Most of the estimators mentioned above are either not efficient (i.e. exponential time complexity) or do not provide non-asymptotic error rates. In contrast, we provide a consistent and nearly linear time algorithm that allows a large fraction of points to be corrupted. Recently, \cite{AnavaHZ2015} studied time series with missing values but their results do not extend to cases with latent corruptions. Moreover, they consider the online setting as compared to the stochastic setting considered by our method.

{\bf Robust Regression}: The goal in RLSR is to recover a parameter using noisy linear observations that are corrupted sparsely. RLSR is a classical problem in statistics, but computationally efficient, provable algorithms have been proposed only in recent years. The Least Trimmed Squares (LTS) method guarantees consistency but in general requires exponential running time \cite{Rousseeuw1984, Visek06a, Visek06b}. Recently \cite{WrightM10,NguyenT13} proposed $L_1$ norm minimization based methods for RLSR but their analyses do not guarantee consistent estimates in presence of dense unbiased i.i.d. noise. Recently, \cite{BhatiaJK2015} proposed a hard thresholding style algorithm for RLSR but are unable to guarantee better than $O(\sigma)$ error in the estimation of $\wo$ where $\sigma$ is the standard deviation of noise. However, as detailed in section \ref{sec:rreg}, their results holds in a weaker adversarial model than ours. In contrast, we provide nearly optimal $\sigma\frac{\sqrt{d}}{\sqrt{n}}$ error rates for our algorithm. \cite{ChenCM2013} considers a stronger model where along with the response variables, the covariates can also be corrupted. However, their result also do not provide consistency guarantees and they can only tolerate $k \leq n/\sqrt d$ corruptions.

\section{Robust Least Squares Regression}
\label{sec:rreg}
We use robust least squares regression (RLSR) as a warm up problem to introduce the tools, as well as establish notation that will be used for our time-series analysis. We present the problem formulation, propose our \myalgo algorithm, and then prove its consistency and robustness guarantees.

\textbf{Problem Formulation and Notation}: We are given a set of $n$ data points $X = \bs{\vx_1,\vx_2,\ldots,\vx_n} \in \bR^{d \times n}$, where $\vx_i \in \bR^d$ are the \emph{covariates}, $\vy \in \bR^n$ is the vector of \emph{responses} generated as
\begin{equation}\label{eq:rr_mod}
\vy = X^\top\wo + \bo + \veps,
\end{equation}
for some \emph{true} underlying model $\wo \in \bR^d$. The responses suffer two kinds of perturbations -- \emph{dense white noise} $\epsilon_i \sim \cN(0,\sigma^2)$ that is chosen in an i.i.d. fashion independently of the data $X$ and the model $\wo$, and \emph{sparse adversarial corruptions} in the form of $\vb$ whose support is chosen independently of $X, \wo$ and $\veps$. We assume that $\bo$ is a $\ko$-sparse vector albeit one with potentially unbounded entries. The constant $\ko$ will be called the \emph{corruption index} of the problem. The above model is stronger than that of \cite{BhatiaJK2015} which considers a fully adaptive adversary. However, whereas \cite{BhatiaJK2015} is unable to give a consistent estimate, we give an algorithm \myalgo that does provide a consistent estimate. We also note that \cite{BhatiaJK2015} is unable to give consistent estimates even in our model. As noted in the next section, our result requires significantly more fine analysis; standard $\ell_2$-norm style anlaysis by \cite{BhatiaJK2015} seems unlikely to lead to a consistency result in the robust regression setting. 

We will require the notions of \emph{Subset Strong Convexity} and \emph{Subset Strong Smoothness} similar to \cite{BhatiaJK2015} and reproduce the same below. For any set $S \subset [n]$, let $X_S := \bs{\vx_i}_{i \in S} \in \bR^{p \times \abs{S}}$ denote the matrix with columns in that set. We define $\vv_S$ for a vector $\vv \in \bR^n$ similarly. $\lambda_{\min}(X)$ and $\lambda_{\max}(X)$ will denote, respectively, the smallest and largest eigenvalues of a square symmetric matrix $X$.

\begin{definition}[SSC and SSS Properties]
\label{defn:ssc-sss}
A matrix $X \in \bR^{p\times n}$ is said to satisfy the \emph{Subset Strong Convexity Property} (resp. \emph{Subset Strong Smoothness Property}) at level $k$ with strong convexity constant $\lambda_k$ (resp. strong smoothness constant $\Lambda_k$) if the following holds:
\[
\lambda_k \leq \underset{|S| = k}{\min} \lambda_{\min}(X_SX_S^\top) \leq \underset{|S| = k}{\max} \lambda_{\max}(X_SX_S^\top) \leq \Lambda_k.
\]
\end{definition}
We refer the reader to the appendix for SSC/SSS bounds for Gaussian ensembles.

\subsection{\myalgo: A Hard Thresholding Approach to Consistent Robust Regression}
We now present our consistent method \myalgo for the RLSR problem. \myalgo takes a significantly different approach to the problem than previous works. Instead of attempting to exclude data points deemed unclean, \myalgo concentrates on correcting the errors instead. This allows \myalgo to work with the entire data set at all times, as opposed \torrent \cite{BhatiaJK2015} that work with a fraction of the data.

Starting with the RLSR formulation $\min_{\vw \in \bR^p, \norm{\vb}_0 \leq \ko} \frac{1}{2}\norm{X^\top\vw - (\vy - \vb)}_2^2$, we realize that given any estimate $\hat\vb$ of the corruption vector, the optimal model with respect to this estimate is given by the expression $\hat\vw = (XX^\top)^{-1}X(\vy - \hat\vb)$. Plugging this expression for $\hat\vw$ into the formulation allows us to reformulate the RLSR problem.
\begin{equation}
\min_{\norm{\vb}_0 \leq \ko} f(\vb) = \frac{1}{2}\norm{(I-P_X)(\vy - \vb)}_2^2
\label{eq:rreg-form-new}
\end{equation}
where $P_X = X^\top(XX^\top)^{-1}X$. This greatly simplifies the problem by casting it as a \emph{sparse parameter estimation} problem instead of a data subset selection problem. \myalgo directly optimizes \eqref{eq:rreg-form-new} by using a form of iterative hard thresholding. At each step, \myalgo performs the following update: $\btn = \HT_{k}(\bt - \nabla f(\bt))$, 
where $k$ is a parameter for \myalgo. Any value $k \geq \ko$ suffices to ensure convergence and consistency. The hard thresholding operator is defined below.
\begin{definition}[Hard Thresholding]
\label{defn:ht}
For any $\vv \in \bR^n$, let the permutation $\sigma_{\vv} \in S_n$ order elements of $\vv$ in descending order of their magnitudes. Then for any $k \leq n$, we define the hard thresholding operator as $\hat\vv = \HT_k(\vv)$ where $\hat\vv_i = \vv_i$ if $\sigma_\vv^{-1}(i) \leq k$ and 0 otherwise.
\end{definition}

We note that \myalgo functions with a fixed, unit step length, which is convenient in practice as it avoids step length tuning, something most IHT algorithms \cite{GargK2009,JainTK2014} require.
For the RLSR problem, we will consider data sets that are Gaussian ensembles i.e. $\vx_i \sim \cN(\vzero,\Sigma)$. Since \myalgo interacts with the data only using the projection matrix $P_X$, one can assume , without loss of generality, that the data points are generated from a standard Gaussian i.e. $\vx_i \sim \cN(\vzero,I_{d\times d})$. Our analysis will take care of the condition number of the data ensemble whenever it is apparent.


\begin{figure*}[t]
\begin{minipage}[t]{0.4\linewidth}
\begin{algorithm}[H]
	\caption{\small \myalgo: Consistent Robust Regression}
	\label{algo:myalgo}
	\begin{algorithmic}[1]
		\small{
			\REQUIRE Covariates $X = \bs{\vx_1,\ldots,\vx_n}$, responses $\vy = [y_1,\ldots,y_n]^\top$, corruption index $k$, tolerance $\epsilon$
			\STATE $\vb^0 \< \vzero, t \< 0,$\\$P_X \< X^\top(XX^\top)^{-1}X$
			\WHILE{$\norm{\bt - \btp}_2 > \epsilon$}
				\STATE $\btn \< \HT_{k}(P_X\bt + (I - P_X)\vy)$%
				\STATE $t \< t + 1$
			\ENDWHILE
			\STATE \textbf{return} $\wt \< (XX^\top)^{-1}X(\vy - \bt)$
		}
	\end{algorithmic}
\end{algorithm}
\end{minipage}
\hskip1.5ex
\begin{minipage}[t]{0.58\linewidth}
\begin{algorithm}[H]
	\caption{\small \aoard: Consistent Robust Time Series Estimation}
	\label{algo:aoard}
	\begin{algorithmic}[1]
		\small{
			\REQUIRE Time-series data $y_i, i=-d+1,\ldots,n$, corruption index $k$, tolerance $\epsilon$, time series order $d$, error trimming level $\hat\sigma$%
			\STATE $y_i = \max\bc{\min\bc{y_i,\hat\sigma},-\hat\sigma}$
			\STATE $\vx_i \< (y_{i-1},\ldots,y_{i-d})^\top$, $X \< \bs{\vx_1,\ldots,\vx_n}$, $\vy \< \br{y_1,\ldots,y_n}^\top$, $P_{X} \< X^\top (XX^\top)^{-1} X$, $t \< 0$, $\vb^0 \< 0$%
			\WHILE{$\norm{\bt - \btp}_2 > \epsilon$}
			\STATE $\btn \< \HT_k^{\cG} \br{P_{X} \bt + \br{I - P_{X}} \vy}$
			\STATE $t \< t+1$
			\ENDWHILE
			\STATE \textbf{return} $\wt \< (XX^\top)^{-1}X(\vy - \bt)$
		}
	\end{algorithmic}
\end{algorithm}
\end{minipage}
\end{figure*}

\subsection{Convergence and Consistency Guarantees}
\begin{theorem}
\label{thm:crr-final}
Let $x_i\in \mathbb{R}^d, 1\leq i\leq n$ be generated i.i.d. from a Gaussian distribution and let $y_i$'s be generated using \eqref{eq:rr_mod} for a fixed $\wo$ and let $\sigma^2$ be the noise variance. Let the number of corruptions $\ko$ be s.t. $\ko \leq k \leq n/10000$. Then, with probability at least $1-\delta$, \myalgo, after $\cO(\log(\norm{\bo}_2/n) + \log(n/(\sigma\cdot d)))$ steps, ensures that $\norm{\wt - \wo}_2 \leq \cO(\sigma\sqrt{d/n\log(nd/\delta}))$.
\end{theorem}
The above result establishes consistency of the \myalgo method with $\tilde\cO(\sigma\sqrt{d/n})$ error rates that are known to be statistically optimal, notably in the presence of gross and unbounded outliers. We reiterate that to the best of our knowledge, this is the first instance of a poly-time algorithm being shown to be consistent for the RLSR problem. It is also notable that the result allows the corruption index to be $\ko = \Omega(n)$, i.e. allows upto a \emph{constant} factor of the total number of data points to be arbitrarily corrupted, while ensuring consistency, which existing results \cite{BhatiaJK2015,NguyenT13} do not ensure.

For our analysis, we will divide \myalgo's execution into two phases -- a \emph{coarse convergence} phase and a \emph{fine convergence} phase. \myalgo will enjoy a linear rate of convergence in both phases. However, the coarse convergence analysis will only ensure $\norm{\wt - \wo}_2 = \bigO{\sigma}$. The fine convergence phase will then use a much more careful analysis of the algorithm to show that in at most $\bigO{\log n}$ more iterations, \myalgo ensures $\norm{\wt - \wo}_2 = \tilde\cO(\sigma\sqrt{d/n})$, thus establishing consistency of the method. Existing methods, including \torrent, are able to reach an error level $\bigO{\sigma}$, but no further.

Let $\vlt := (XX^\top)^{-1}X(\bt - \bo)$, $\vg := (I - P_X)\veps$, and $\vv^t = X^\top\vlt + \vg$. Let $S^\ast := \supp(\bo)$ true locations of the corruptions and $I^t := \supp(\bt) \cup \supp(\bo)$. Let $\mdt = \supp(\bo)\backslash\supp(\bt)$, $\fat = \supp(\bt)\backslash\supp(\bo)$, and $\cit = \supp(\bt)\cap\supp(\bo)$ respectively denote the coordinates that were \emph{missed detections}, \emph{false alarms}, and \emph{correctly identifications}.

\textbf{Coarse convergence}: Here we establish a result that guarantees that after a certain number of steps $T_0$, \myalgo identifies the corruption vector with a relatively high accuracy i.e. $\norm{\vw^{T_0} - \wo}_2 \leq \bigO{\sigma}$.

\begin{lemma}
\label{lem:coarse-conv}
For any data matrix $X$ that satisfies the SSC and SSS properties such that $\frac{2\Lambda_{k+\ko}}{\lambda_n} < 1$, \myalgo, when executed with a parameter $k \geq \ko$, ensures that after $T_0 = \bigO{\log\frac{\norm{\bo}_2}{\sqrt n}}$ steps, $\norm{\vb^{T_0} - \bo}_2 \leq 3e_0$, where $e_0 = \bigO{\sigma\sqrt{(k+\ko)\log\frac{n}{\delta(k+\ko)}}}$ for standard Gaussian designs.
\end{lemma}

Using Lemma~\ref{lem:w-lambda-link} (see the appendix), we can translate the above result to show that $\norm{\vw^{T_0} - \wo}_2 \leq 0.95\sigma$, assuming $k = \ko \leq \frac{n}{150}$. However, Lemma~\ref{lem:coarse-conv} will be more useful in the following analysis.


\textbf{Fine convergence}: We now show that \myalgo progresses further at a linear rate to achieve a consistent solution. First Lemma~\ref{lem:lambda-fa-link} will show that $\norm{\vlt}_2$ can be bounded, apart from diminishing or negligible terms, by the amount of mass that is present in the false alarm coordinates $\mdt$. Lemma~\ref{lem:lambda-bound} will next bound this quantity. For all analyses hereon, we will assume $t > T_0$.


\begin{lemma}
\label{lem:lambda-fa-link}
Suppose $\ko \leq k \leq n/10000$. Then with probability $1-\delta$, at every time instant $t > T_0$, \myalgo ensures that $\|\vln\|_2 \leq \frac{1}{100}\|\vlt\|_2 + 2\sigma\sqrt{\frac{2d}{n}\log\frac{d}{\delta}} + \frac{2.001}{\lambda_n}\|X_\fan(X_\fan^\top\vlt + \vg_\fan)\|_2$.
\end{lemma}

We note that in the RHS above, the first term diminishes at a linear rate and the second term is a negligible quantity since it is $\tilde\cO(\sqrt{d/n})$. In the following we bound the third term.

\begin{lemma}
\label{lem:lambda-bound}
For $\ko \leq k \leq n/10000$, with probability at least $1 - \delta$, \myalgo ensures at all $t > T_0$, $\frac{2.001}{\lambda_n}\norm{X_\fan(X_\fan^\top\vlt + \vg_\fan)}_2 \leq 0.98\norm{\vlt}_2 + C\cdot\sigma\sqrt{\frac{d}{n}\log\frac{nd}{\delta}}$ for some constant $C$.
\end{lemma}

Putting all these results together establishes Theorem~\ref{thm:crr-final}. See Appendix~\ref{app:rreg} for a detailed proof.


\section{Robust Time Series Estimation}
\label{sec:ts}
Similar to RLSR, we formulate the Robust Time Series Estimation (RTSE)  with additive outliers (AO) problem, propose our \aoard algorithm, and prove its consistency and robustness guarantees.

\paragraph{Problem Formulation and Notation:}
Let $(x_{-d+1},\ldots,x_n)$ be the ``clean'' time series which is a stationary and stable $\ARD$ process defined as $x_t = x_{t-1} \wo_1 + \cdots + x_{t-d} \wo_d + \veps_t$ where $\veps_t \sim \cN(0,\sigma^2)$ are i.i.d. noise values chosen independently of the data and the model. We compactly represent this $\ARD$ process as, 
\[
\yo = \bar{X}^\top \wo + \veps ,
\]
where $\yo = \br{x_1,\ldots,x_n}^\top \in \bR^n,\ \vx_i = (x_{i-1},\ldots,x_{i-d})^\top$, and $\bar{X} = \bs{\vx_1,\ldots,\vx_n} \in \bR^{d \times n}$.
However, we do not observe the ``clean'' time series. Instead, we observe the time series $(y_{-d+1},\ldots,y_n)$ which contains additive corruptions. 
Defining $\vy \in \bR^n$, $X \in \bR^{d \times n}$  using $(y_{-d+1},\ldots,y_n)$ in similar manner as $\yo$ and $\bar{X}$ are defined using  $(x_{-d+1},\ldots,x_n)$, we have the resulting AO model as follows: 
\begin{equation}
\label{ao-corruption-model-specification}
\vy = \yo + \eo = X^\top \wo + \veps + \bo ,
\end{equation}
where $\eo$ is the actual corruption vector ($\ko$-sparse), and $\bo$ is the resulting model corruption vector (with at most $\ko$-blocks of size $d$ being non-zero). See \eqref{design-matrix-connection} (see Appendix~\ref{time-series-notes}) for a clearer characterization of the $\vy, X$.

Now, given $\vy, X$, our goal will be to recover a consistent estimate of the parameter $\wo$. For our results the following simple observation would be crucial: since $\supp(\bo)$ is a union of $\ko$ groups (intervals) of size $d$,	we have $\norm{\bo}_0^{\cG} \leq 2\ko$, where $\norm{\vb}_0^{\cG}$ is the Group-$\ell_0$ pseudo-norm of $\vb$ that we define below. For a set of groups $S$, $\supp(S;\cG) = \bc{G_i, i \in S}$.

We now define certain quantities that are crucial in understanding the $\ARD$ process. The \emph{spectral density} of the ``clean'' $\ARD$ process $\yo$ is given by: 
\begin{equation}
\label{spectral-desity-function}
\rho_{\wo}\br{\omega} = \frac{\sigma^2}{\br{1-\sum_{k=1}^{d}{\wo_k e^{ik\omega}}} \br{1-\sum_{k=1}^{d}{\wo_k e^{-ik\omega}}}}, \text{ for } \omega \in \bs{0,2\pi}.
\end{equation}
We define $\cM_{\wo} := \sup_{\omega \in \bs{0,2\pi}} \rho_{\wo}\br{\omega}$ and $\mathfrak{m}_{\wo} := \inf_{\omega \in \bs{0,2\pi}} \rho_{\wo}\br{\omega}$. Another constant $\cM_{W}$ will also appear in our results (see Appendix~\ref{time-series-notes} for a brief primer on $\ARD$ process). 


For our analysis, we will also require notions of \emph{Sub-group Strong Convexity} and \emph{Sub-group Strong Smoothness} for the time series which we define below. For any $k \leq \frac{n}{d}$, we let $\cS_k^{\cG} = \bc{\supp(S;\cG) : S \subseteq \bs{\frac{n}{d}} \text{ s.t. } \abs{S}=k}$ denote the {\em set of all} collections of $k$ groups from $\cG$.
\begin{definition}[SGSC/SGSS]
	A matrix $X \in \bR^{d \times n}$ satisfies the Subgroup Strong Convexity Property
	(resp. Subgroup Strong Smoothness Property) at level $k$ with strong convexity constant $\lambda_k$ (resp. strong smoothness constant $\Lambda_k$) if the following holds:
	\[
	\lambda_k ~\leq~ \min_{S \in \cS_k^{\cG}} \lambda_{\min}\br{X_S X_S^\top } ~\leq~ \max_{S \in \cS_k^{\cG}} \lambda_{\max}\br{X_S X_S^\top} ~\leq~ \Lambda_k .
	\]
\end{definition}

\subsection{\aoard: A Block Sparse Hard Thresholding Approach to Consistent Robust Time Series Estimation}

We now present our \aoard method for obtaining consistent estimates in the RTSE problem. By following the similar approach as \myalgo, we begin with the RTSE formulation $\min_{\vw \in \bR^d, \norm{\vb}_0^{\cG} \leq \ko} \frac{1}{2}\norm{X^\top\vw - (\vy - \vb)}_2^2$, and observe that for any given estimate $\hat\vb$ of the corruption vector, the optimal model with respect to that estimate is $\hat\vw = (XX^\top)^{-1}X(\vy - \hat\vb)$. Then by plugging this expression for $\hat\vw$ into the formulation, we reformulate the RTSE problem as follows
\begin{equation}
\min_{\norm{\vb}_0^{\cG} \leq \ko} f(\vb) = \frac{1}{2}\norm{(I-P_X)(\vy - \vb)}_2^2
\label{eq:rtse-form-new}
\end{equation}
where $P_X = X^\top(XX^\top)^{-1}X$. \aoard uses a variant of iterative hard thresholding to optimize the above formulation. At every iteration, \aoard takes a step along the negative gradient of the function $f$ and then performs group hard thresholding to select the \emph{top} $k$ aligned groups (i.e. groups in $\cG$) of the resulting vector and setting the rest to zero.
\[
\btn = \HT_{k}^{\cG}(\bt - \nabla f(\bt)) ,
\]
where $k \geq 2\ko$ and the group hard thresholding operator is defined below. 
\begin{definition}[Group Hard Thresholding]\label{defn:ght}
For any vector $\vg \in \bR^n$, let $\sigma_{\vg} \in S_{\frac{n}{d}}$ be the permutation s.t. $\sum_{j \in G_{\sigma_{\vg}(1)}}{\abs{{\vg}_j}^2} \geq \sum_{j \in G_{\sigma_{\vg}(2)}}{\abs{{\vg}_j}}^2 \geq \ldots \geq \sum_{j \in G_{\sigma_{\vg}(\frac{n}{d})}}{\abs{{\vg}_j}}^2$. Then for any $k \leq \frac{n}{d}$, we define the group hard thresholding operator as $\hat{{\vg}}=\HT_k^{\cG}({\vg})$ where 
	\[
	\hat{{\vg}}_i ~=~ \begin{cases}
	{\vg}_i &\text{if $\sigma_{\vg}^{-1}(\ceil{\frac{i}{d}}) \leq k$}\\
	0 &\text{else}
	\end{cases}
	\]
\end{definition}
We note that this step can be done in quasi linear time. Due to the delicate correlations between data points in the time series, in order to keep the problem well conditioned (see Theorem~\ref{ao-ssc-sss} and Remark~\ref{rough-bounds}), we will perform a pre-processing step on the corrupted time series instances $y_i, i=-d+1,\ldots,n$ as follows: $y_i = \max\bc{\min\bc{y_i,\hat\sigma},-\hat\sigma}$, where $\hat{\sigma} = \bigO{\sqrt{\log{n}}\sigma}$. Note that since the clean underlying time series is a Gaussian process $\veps_i \leq \bigO{\sigma\sqrt{\log{n}}}$ and all its entries are, with high probability, bounded by $\hat\sigma$. Thus we will not clip any clean point because of the above step but ensure that we can, from now on, assume that $\norm{\bo}_\infty \leq \hat\sigma$. 


\subsection{Convergence and Consistency Guarantees}
We now present the estimation error bound for our \aoard algorithm. 
\begin{theorem}
	\label{thm:crtse-final}
Let $\vy$ be generated using $\ARD$ process with $\ko$ additive outliers (see \eqref{ao-corruption-model-specification}). Also, let $\ko \leq k \leq C \frac{\mathfrak{m}_{\wo}}{\cM_{\wo} + \cM_{W}}\frac{n}{d \log{n}}$ (for some universal constant $C > 0$). Then, with probability at least $1-\delta$, \aoard, after $\cO(\log(\norm{\bo}_2/n) + \log(n/(\sigma\cdot d)))$ steps, ensures that $\norm{\wt - \wo}_2 \leq \bigO{\sigma \cM_{\wo}/\mathfrak{m}_{\wo} \sqrt{d\log{n}/n \log\br{d/ \delta}}}$.
\end{theorem}
The result does establish consistency of the \aoard method as it offers convergence to $\softO{\sigma\sqrt{d \log{n} / n}}$ error levels. Also note that in typical time series data, $d$ lies in the range $5-10$. As in the case of \myalgo, this is the first instance of a poly-time algorithm being shown to be consistent for the RTSE problem. 

Following the similar approach of the consistency analysis for \myalgo, we will first ensure that $\norm{\wt - \wo}_2=\bigO{\sigma}$. Then in the fine analysis phase, we will show that after additional $\bigO{\log{n}}$ iterations, \aoard ensures $\norm{\wt - \wo}_2=\softO{\sigma\sqrt{d \log{n} / n}}$.

\textbf{Coarse convergence}: Here we establish a result that after a certain number of iterations, \aoard identifies the corruption vector with a relatively high accuracy. Our analysis relies on a novel Theorem~\ref{ao-ssc-sss}, which is a {\em key result} that shows that the $\ARD$ process with AO indeed satisfies SGSC and SGSS properties (see Definition~\ref{defn:ght}), as long as the number of corruptions $k^*$ is small. 
\begin{theorem}
	\label{ao-ard-coarse-convergence}
	For any data matrix $X$ that satisfies the SGSC and SGSS properties such that $4 \Lambda_{k+\ko} < \lambda_{\frac{n}{d}}$, \aoard, when executed with a parameter $k \geq \ko$, ensures that after $T_0 = \bigO{\log\br{\norm{\bo}_2 / \sqrt n}}$ steps, $\norm{\vb^{T_0} - \bo}_2 \leq 5e_0$. Additionally, if $X$ is generated using our $\ARD$ process with AO (see \eqref{ao-corruption-model-specification}), then $e_0 = \cO\br{\sigma\sqrt{(k+\ko)d\log\frac{n}{\delta(k+\ko)d}}}$.
\end{theorem}
Note that if $X$ is given by $\ARD$ process with AO model and if $k$ is sufficiently small i.e. $\ko \leq k \leq C \frac{\mathfrak{m}_{\wo}}{\cM_{\wo} + \cM_{W}}\frac{n}{d \log{n}}$ (for some universal constant $C > 0$) and $n$ is sufficiently large enough, then with probability at least $1-\delta$, we have $4 \Lambda_{k+\ko} < \lambda_{\frac{n}{d}}$. See Remark~\ref{rough-bounds} for more details.


\paragraph{Fine Convergence:}
As was the case in least squares regression, we will now sketch a proof that the \aoard algorithm indeed moves beyond the convergence level achieved in the coarse analysis and proceeds towards a consistent solution at a linear rate. We begin by noting that by applying Lemma~\ref{lem:Xe-bound}, we can derive a result similar to Lemma~\ref{lem:w-lambda-link}. With high probability, we have for all $t > 1$
\begin{equation}
\label{ard-consistency-lemma-eq}
\norm{\wt - \wo}_2 \leq C\cdot\frac{\Lambda_n}{\lambda_n}\br{\sigma\sqrt{\frac{d\log n}{n}\log\frac{d}{\delta}} + \norm{\vlt}_2},
\end{equation}
for a universal constant $C$. We note that for large enough $n$, Lemma~\ref{full-singular-values} shows that $\frac{\Lambda_n}{\lambda_n} = \bigO 1$. Since the first term in the bracket is a negligible term, one that does not hinder consistency, save $\log$ factors, we are just left to establish the convergence of the iterates $\vlt$. We next note that Lemma~\ref{lem:Xe-bound}, along with the fact that the locations of the corruptions were decided obliviously and independently of the noise values $\bc{\epsilon_i}$, allows us to also prove the following equivalent of Lemma~\ref{lem:lambda-fa-link} for the time series case as well: with high probability, for every time instant $t > T_0$, we have
\begin{equation}
\label{ard-recurrence-lemma-eq}
\norm{\vln}_2 \leq \frac{1}{100}\norm{\vlt}_2 + C\cdot\br{\sigma\sqrt{\frac{d\log n}{n}\log\frac{d}{\delta}} + \frac{1}{\lambda_n}\br{1 + \frac{\Lambda_n}{\lambda_n}}\norm{X_\fan(X_\fan^\top\vlt + \vg_\fan)}_2},
\end{equation}
for some universal constant $C$. Noticing yet again that $\frac{\Lambda_n}{\lambda_n} = \bigO 1$ leaves us to prove a bound on the quantity $\norm{X_\fan(X_\fan^\top\vlt + \vg_\fan)}_2$. We now notice that one can upper bound this quantity by $\norm{X_\fan(X_{S^t_k}^\top\vlt + \vg_{S^t_k})}_2$ by selecting the set $S^t_k$ of the top $k$ elements by magnitude in the vector $X_{\bar{S^\ast}}^\top\vlt + \vg_{\bar{S^\ast}}$.
This allows us to establish the following result.
\begin{lemma}
\label{lem:lambda-bound-ts}
Suppose $\ko \leq k \leq n/(C'\rho(\wo)d\log n)$ for some large enough constant $C'$. Then with probability at least $1 - \delta$, \myalgo ensures at every time instant $t > T_0$
\[
\frac{C}{\lambda_n}\br{1+\frac{\Lambda_n}{\lambda_n}}\norm{X_\fan(X_\fan^\top\vlt + \vg_\fan)}_2 \leq 0.5\norm{\vlt}_2 + \bigO{\sigma\sqrt{\frac{d\log n}{n}\log\frac{1}{\delta}}}
\]
\end{lemma}
Above lemma with \eqref{ard-recurrence-lemma-eq} suffices to establish Theorem~\ref{thm:crtse-final}. See Appendix~\ref{app:app_add} for details of all the steps sketched above.


\vspace*{-5pt}
\section{Experiments}\vspace*{-5pt}
\label{sec:exps}

Several numerical simulations were carried out on synthetically generated linear regression and $\ARD$ time-series data with outliers. The experiments show that in the robust linear regression setting, \myalgo gives a consistent estimator and is $2$x times faster as compared with TORRENT \cite{BhatiaJK2015} while in the robust $\ARD$ time-series setting, \aoard gives a consistent estimator and offers statistically better recovery properties as compared with baseline algorithms.

\begin{figure}[t!]
  \centering\hspace*{-4ex}
\begin{tabular}{cccc}
  \includegraphics[width=.27\textwidth]{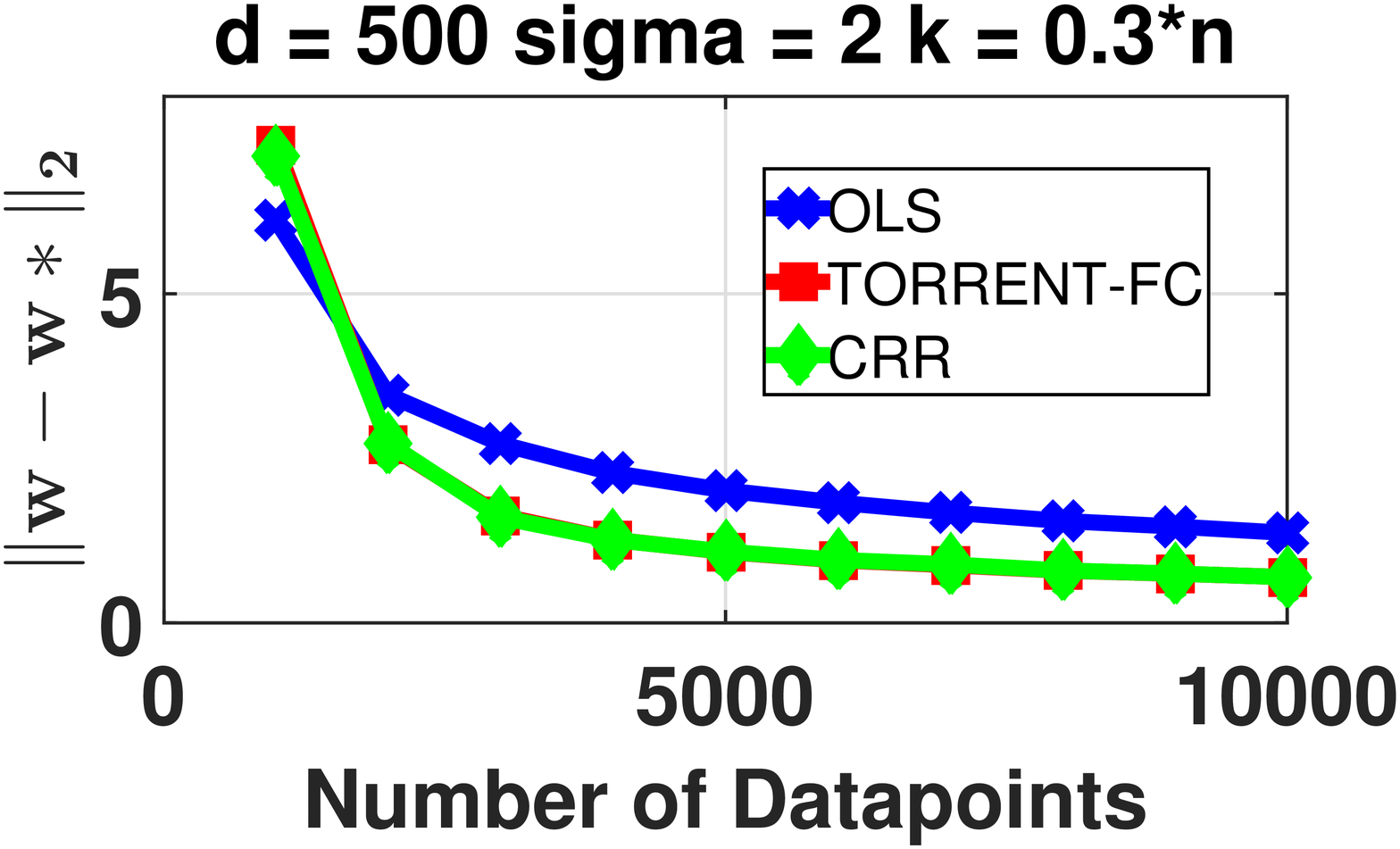}&
  \hspace{-4ex}
  \includegraphics[width=.27\textwidth]{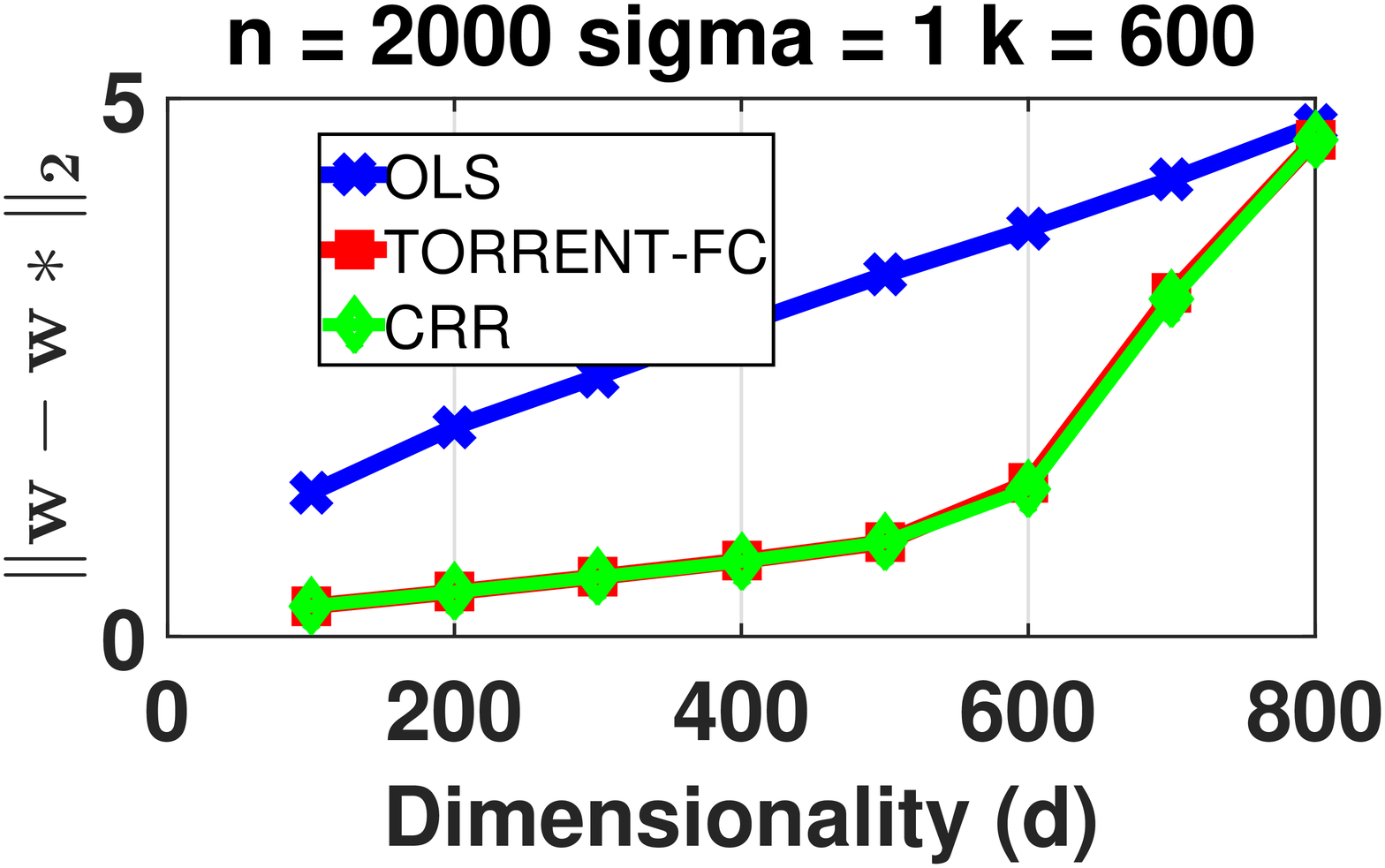}&
  \hspace{-4ex}
  \includegraphics[width=.27\textwidth]{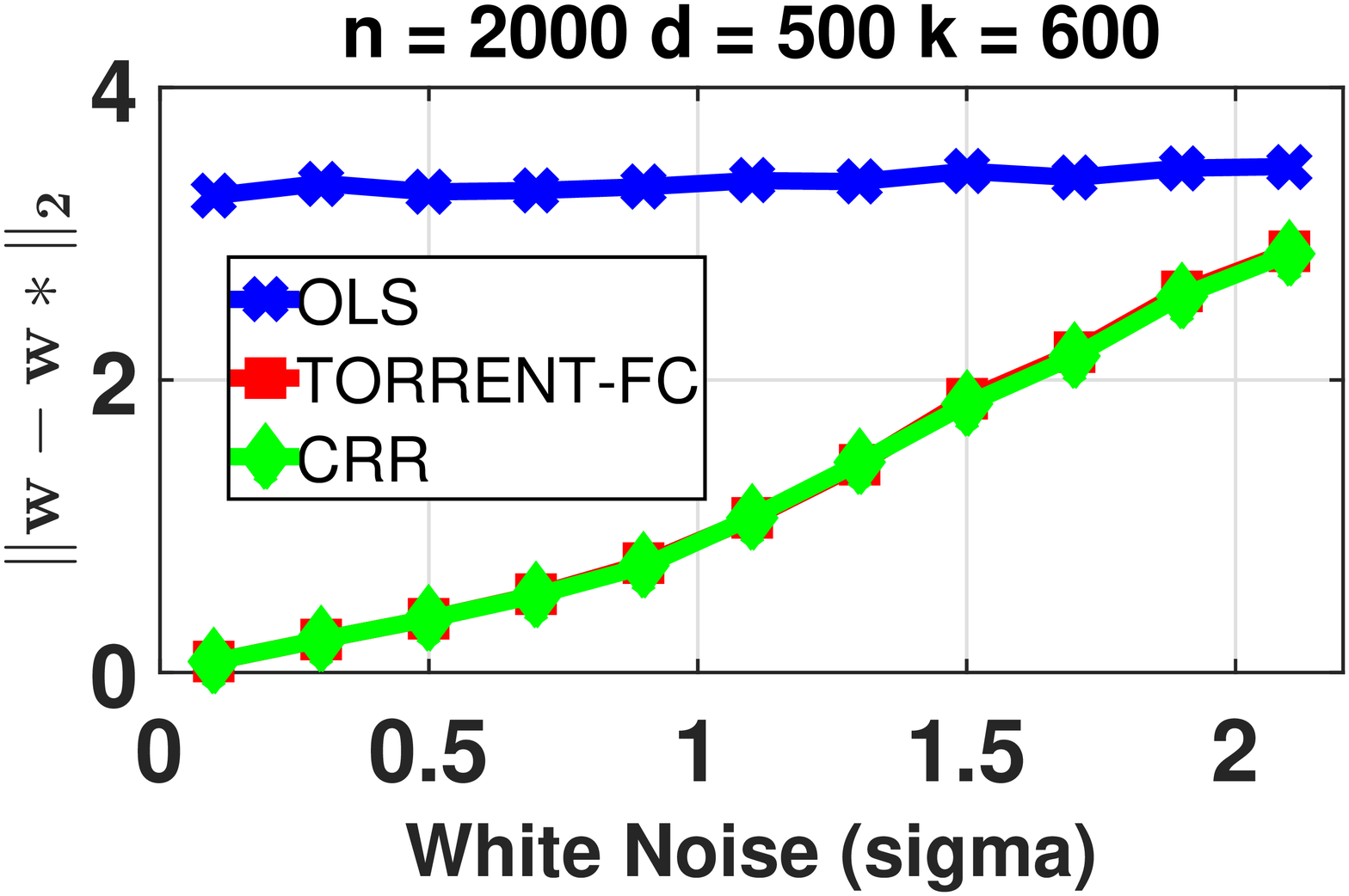}&
  \hspace{-4ex}
  \includegraphics[width=.27\textwidth]{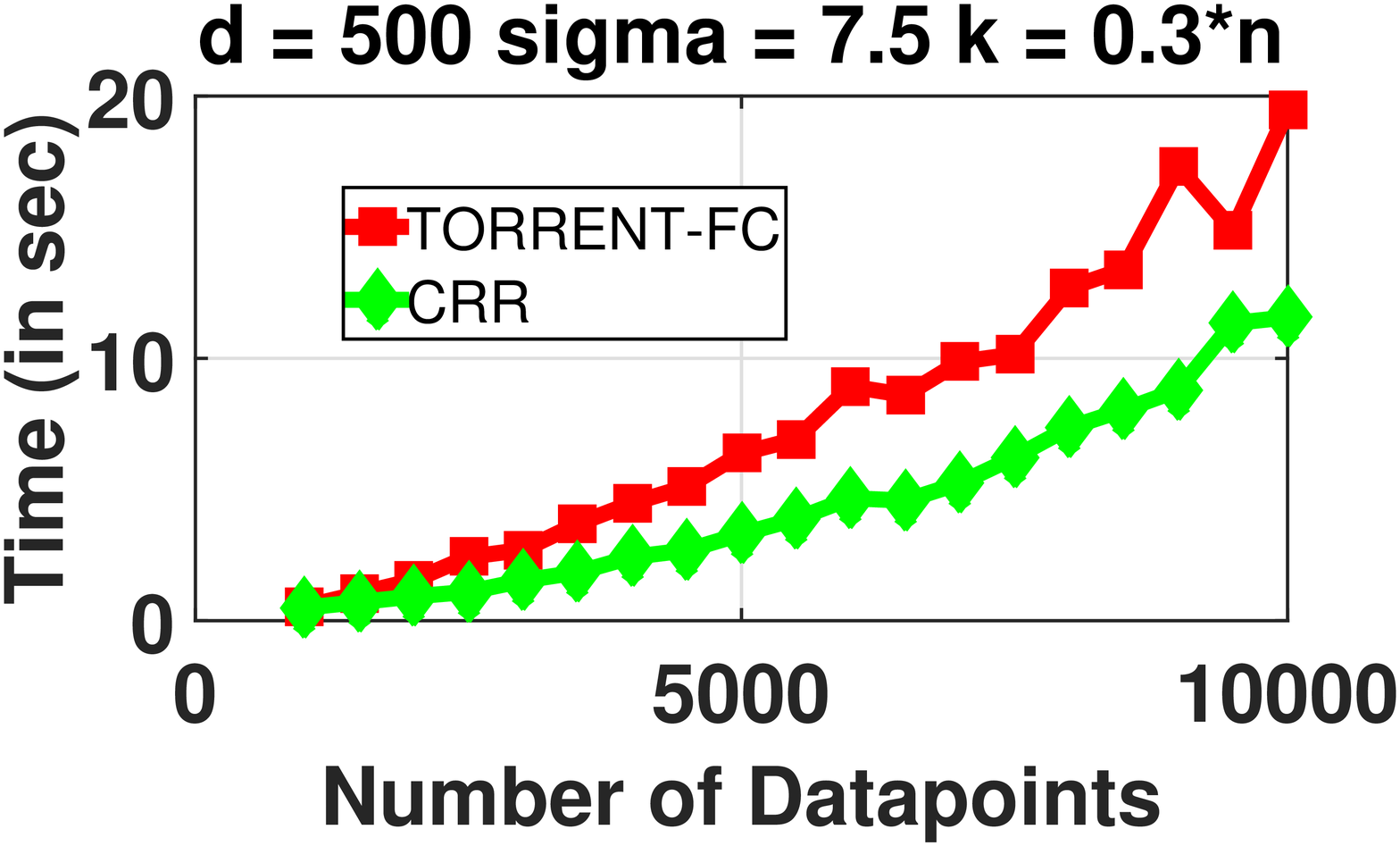}\\
(a)&(b)&(c)&(d)\vspace*{-10pt}
\end{tabular}
\caption{\small{(a), (b) and (c) show variation of recovery error with varying $n$, $d$ and $\sigma$. \myalgo and TORRENT show better recovery properties than the non-robust OLS. These plots also ascertain the $\sqrt{n}$-consistency of \myalgo as is shown in the theoretical analysis. (d) shows the average CPU run time of TORRENT and \myalgo with increasing sample size. \myalgo can be upto $2$x faster than TORRENT while ensuring similar recovery properties.}}\vspace*{-10pt}
  \label{fig:plt_lr}
\end{figure}
\vspace*{-5pt}
\subsection{Robust Linear Regression}\vspace*{-5pt}
\textbf{Data:} For the RLSR problem, the regressor $\wo\in \bR^d$ was chosen to be a random unit norm vector. The data matrix was generated as each $\vx_i \sim  \cN(0, I_d)$.  The $k^\ast$ non-zero locations of the corruption vector $\bo$ were chosen uniformly at random from $[n]$ and the value of the corruptions were set to $b^\ast_i \sim U \br {10, 20}$. The response variables $\vy$ were then generated as $y_i = \ip{\vx_i}{\wo} + \eta_i + b^\ast_i$ where $\eta_i \sim \cN(0, \sigma^2)$. All plots for the RLSR problem have been generated by averaging the results over 20 random instances of the data and regressor.

\textbf{Baseline Algorithms:} We compare \myalgo with two baseline algorithms: Ordinary Least Squares (OLS) and TORRENT (\cite{BhatiaJK2015}). All the three algorithms were implemented in Matlab and were run on a single core 2.4GHz machine with 8GB RAM.

\textbf{Recovery Properties \& Timing:} As can be observed from Figure(\ref{fig:plt_lr}), \myalgo performs as well as TORRENT in terms of the residual error $\| \vw - \wo\|_2$ and both their performances are better as compared with the non-robust OLS method. Further, figures \ref{fig:plt_lr}(a), \ref{fig:plt_lr}(b) and \ref{fig:plt_lr}(c) explain our near optimal recovery bound of $\sigma \sqrt{\frac{d}{n}}$ by showing the corresponding variation of the recovery error with variations in $n$, $d$ and $\sigma$, respectively. Figure \ref{fig:plt_lr}(d) shows a comparison of variation of average CPU time (in secs) with increasing number of data samples and shows that \myalgo can be upto $2$x faster than TORRENT while provably guaranteeing consistent estimates for the regressor.

\begin{figure}[t!]
  \centering\hspace*{-4ex}
\begin{tabular}{cccc}
  \includegraphics[width=.28\textwidth]{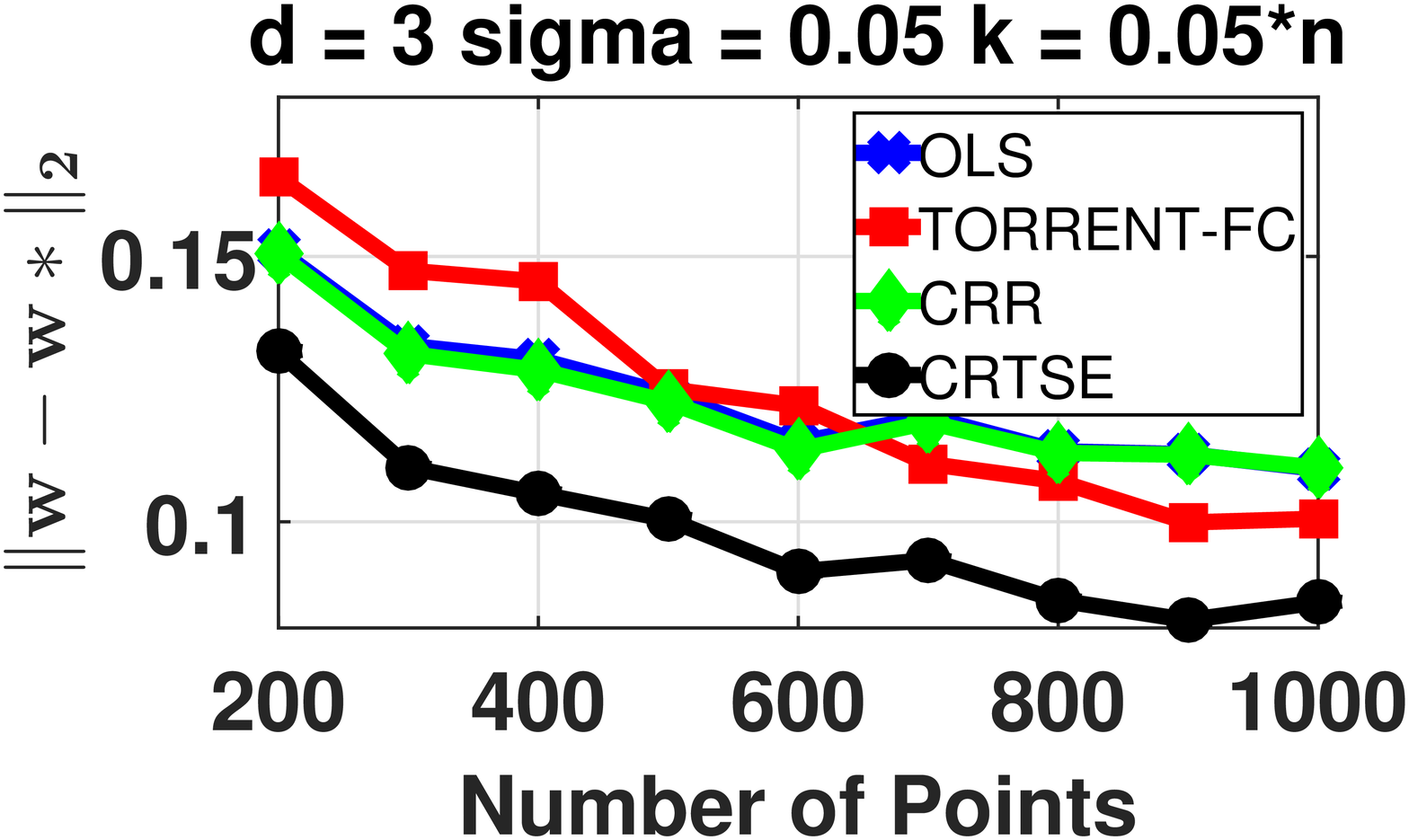}&
  \hspace{-4ex}
  \includegraphics[width=.28\textwidth]{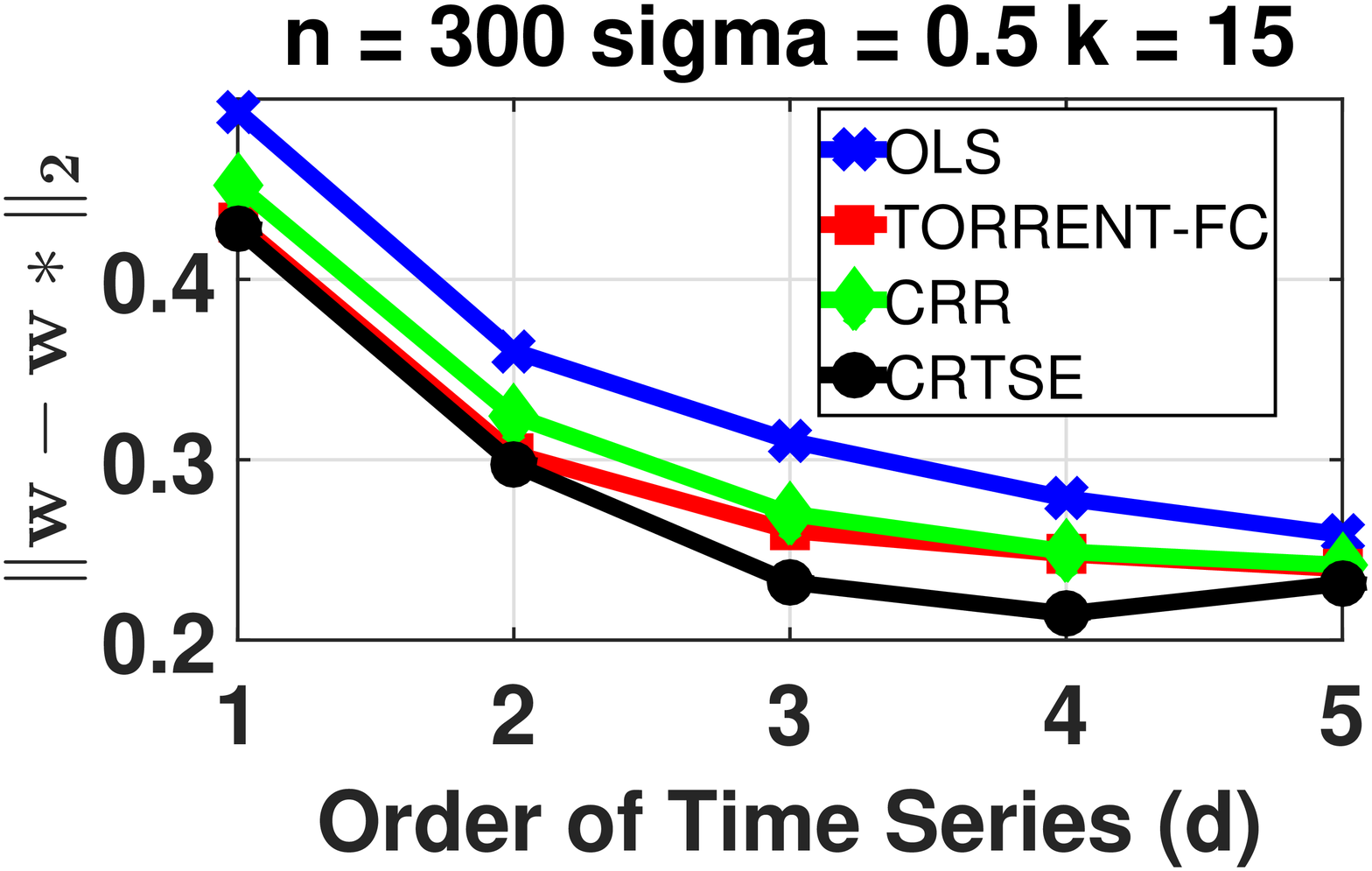}&
  \hspace{-4ex}
  \includegraphics[width=.28\textwidth]{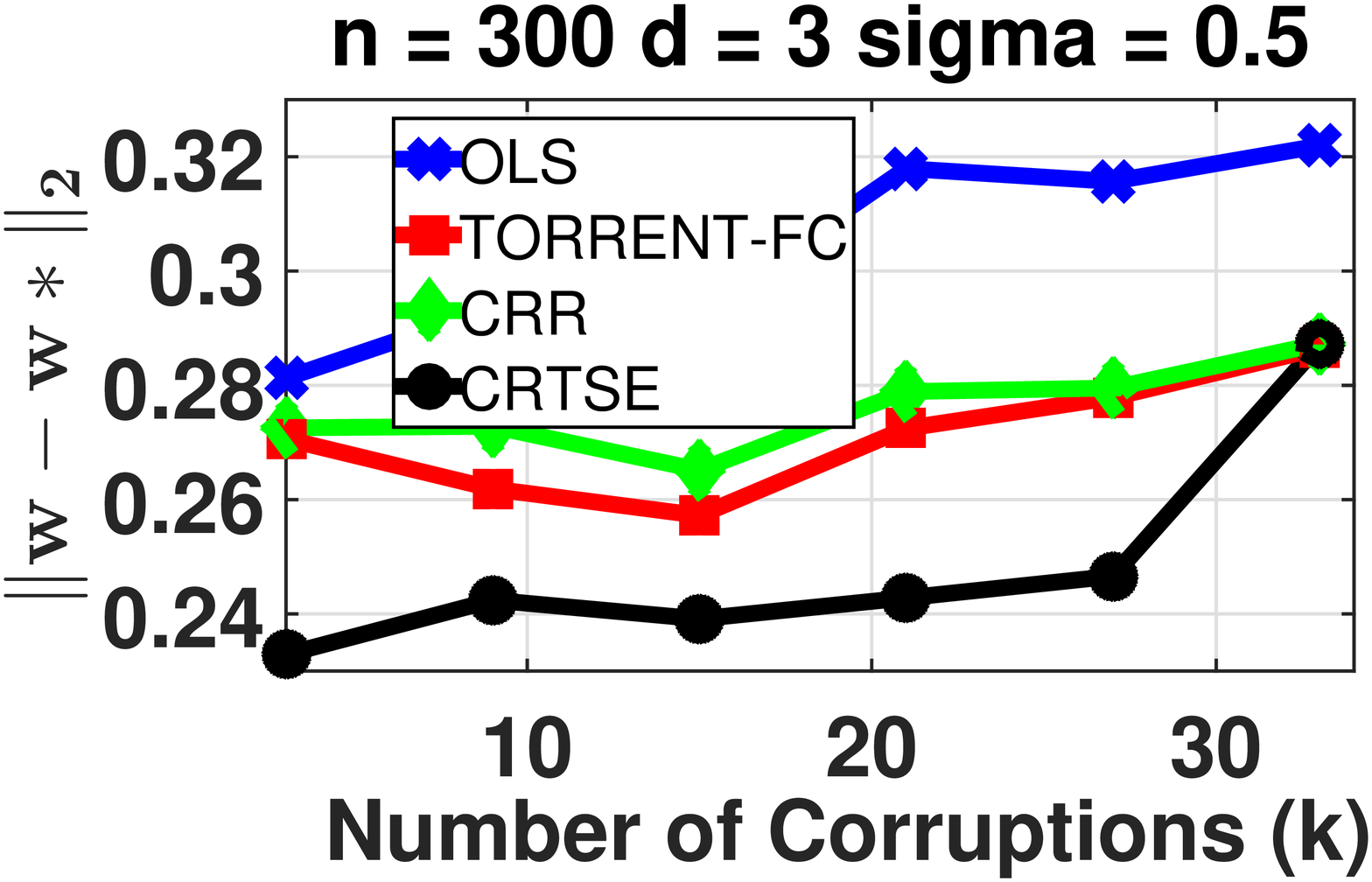}&
  \hspace{-4ex}
  \includegraphics[width=.28\textwidth]{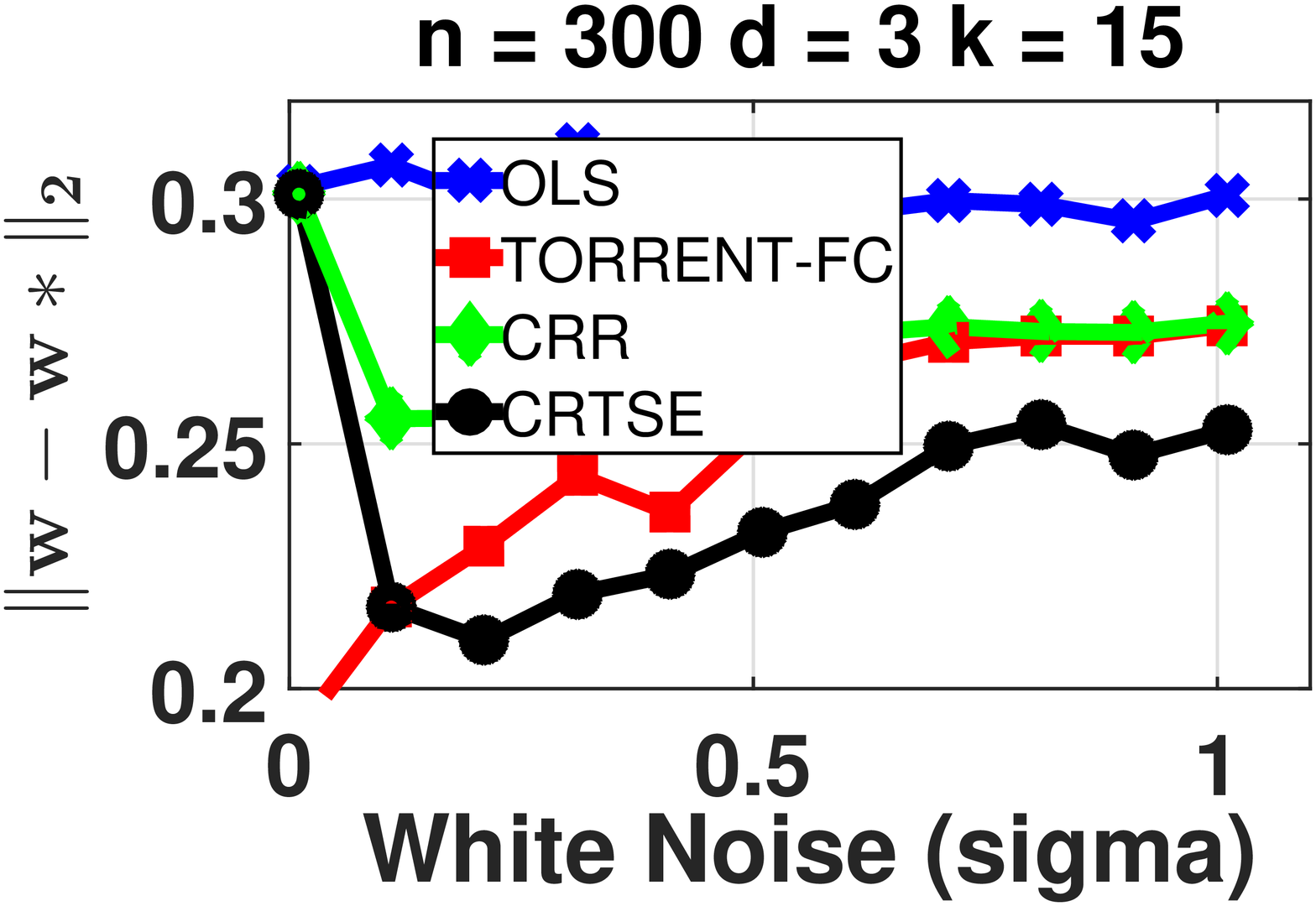}\\
(a)&(b)&(c)&(d)\vspace*{-10pt}
\end{tabular}
\caption{\small{(a), (b), (c), and (d) show variation of recovery error with varying $n$, $d$, $k$ and $\sigma$, respectively. \aoard outperforms OLS, and both the point-wise thresholding algorithms, TORRENT and \myalgo. Also, the decreasing error with increasing $n$ shows the consistency of our estimator in this regime.}}\vspace*{-10pt}
  \label{fig:plt_ts}\vspace*{-5pt}
\end{figure}
\vspace*{-5pt}
\subsection{Robust Time Series with Additive Corruptions}\vspace*{-5pt}
\textbf{Data:} For the RTSE problem, the regressor $\wo\in \bR^d$ was chosen to be a random vector with $O(\frac{1}{\sqrt{d}})$ norm (to avoid the time-series from diverging). The initial $d$ points of the time-series are chosen as $x_i \sim  \cN(0, 1)$ for $i = 1 \ldots d$. The time-series, generated according $\ARD$ model with regressor $\wo$, was then allowed to stabilize for the next $100$ time-steps. We consider the points generated in the next $n$ time steps as $x_i$ for $i = 1 \ldots n$. The $k^\ast$ non-zero locations of the corruption vector $\bo$ were chosen uniformly at random from $[n]$ and the value of the corruptions were set to $b^\ast_i \sim U \br {10, 20}$. The observed time series is then generated as $y_i = x_i + b^\ast_i$. All plots for the RTSE problem have been generated by averaging the outcomes over 200 random runs of the above procedure. 

\textbf{Baseline Algorithms:} We compare \aoard with three baseline algorithms: Ordinary Least Squares (OLS) , TORRENT (\cite{BhatiaJK2015}) and \myalgo. For TORRENT and \myalgo, we set the thresholding parameter $k = 2k^\ast d$ and compare results with \aoard. All simulations were done on a single core 2.4GHz machine with 8GB RAM.

\textbf{Recovery Properties:} Figure \ref{fig:plt_ts} shows the variation of recovery error $\|\vw - \wo\|_2$ for the $\ARD$ time-series with Additive Corruptions. \aoard outperforms all three competitor baselines: OLS, TORRENT and \myalgo. Since \aoard uses a group thresholding based algorithm as compared with TORRENT and \myalgo which use point-wise thresholding, \aoard is able to identify blocks which contain both response and data corruptions and give better estimates for the regressor. Also, figure \ref{fig:plt_ts}(a) shows that the recovery error goes down with increasing number of points in the time-series, as is evident from our consistency analysis of \aoard.

\newpage
{\small
\bibliographystyle{plain}
\bibliography{refs}
}

\clearpage

\appendix

\section{Supplementary Material for Consistent Robust Regression}\label{app:rreg}

\subsection{SSC/SSS guarantees}

In this section we restate some results from \cite{BhatiaJK2015} which are required for the convergence analysis of the RLSR problem.

\begin{definition}
	A random variable $x \in \bR$ is called sub-Gaussian if the following quantity is finite
	\[
	\underset{p \geq 1}{\sup}\ p^{-1/2}\br{\E{\abs{x}^p}}^{1/p}.
	\]
	Moreover, the smallest upper bound on this quantity is referred to as the sub-Gaussian norm of $x$ and denoted as $\norm{x}_{\psi_2}$.
\end{definition}

\begin{definition}
	A vector-valued random variable $\vx \in \bR^p$ is called sub-Gaussian if its unidimensional marginals $\ip{\vx}{\vv}$ are sub-Gaussian for all $\vv \in S^{p-1}$. Moreover, its sub-Gaussian norm is defined as follows
	\[
	\norm{X}_{\psi_2} :=  \underset{\vv \in S^{p-1}}{\sup} \norm{\ip{\vx}{\vv}}_{\psi_2}
	\]
\end{definition}

\begin{lemma}
	\label{lem:ev-bound-gaussian-global}
	Let $X \in \bR^{p \times n}$ be a matrix whose columns are sampled i.i.d from a standard Gaussian distribution i.e. $\vx_i \sim \cN(0,I)$. Then for any $\epsilon > 0$, with probability at least $1 - \delta$, $X$ satisfies
	\begin{align*}
	\lambda_{\max}(XX^\top) &\leq n + (1 - 2\epsilon)^{-1}\sqrt{cnp + c'n\log\frac{2}{\delta}}\\
	\lambda_{\min}(XX^\top) &\geq n - (1 - 2\epsilon)^{-1}\sqrt{cnp + c'n\log\frac{2}{\delta}},
	\end{align*}
	where $c = 24e^2\log\frac{3}{\epsilon}$ and $c' = 24e^2$.
\end{lemma}

\begin{theorem}
	\label{thm:ev-bound-gaussian-local}
	Let $X \in \bR^{p \times n}$ be a matrix whose columns are sampled i.i.d from a standard Gaussian distribution i.e. $\vx_i \sim \cN(0,I)$. Then for any $\gamma > 0$, with probability at least $1 - \delta$, the matrix $X$ satisfies the SSC and SSS properties with constants
	\begin{align*}
	\Lambda_\gamma &\leq \gamma n\br{1 + 3e\sqrt{6\log\frac{e}{\gamma}}} + \bigO{\sqrt{np + n\log\frac{1}{\delta}}} \\
	\lambda_\gamma &\geq n - (1-\gamma)n\br{1 + 3e\sqrt{6\log\frac{e}{1-\gamma}}} - \Omega \br{\sqrt{np + n\log\frac{1}{\delta}}}.
	\end{align*}
\end{theorem}

\begin{lemma}
	\label{lem:ev-bound-subgaussian-global}
	Let $X \in \bR^{p \times n}$ be a matrix with columns sampled from some sub-Gaussian distribution with sub-Gaussian norm $K$ and covariance $\Sigma$. Then, for any $\delta > 0$, with probability at least $1 - \delta$, each of the following statements holds true:
	\begin{align*}
	\lambda_{\max}(XX^\top) &\leq \lambda_{\max}(\Sigma)\cdot n + C_K\cdot\sqrt{pn} + t\sqrt{n}\\
	\lambda_{\min}(XX^\top) &\geq \lambda_{\min}(\Sigma)\cdot n - C_K\cdot\sqrt{pn} - t\sqrt{n},
	\end{align*}
	where $t = \sqrt{\frac{1}{c_K}\log\frac{2}{\delta}}$, and $c_K, C_K$ are absolute constants that depend only on the sub-Gaussian norm $K$ of the distribution.
\end{lemma}

\subsection{Convergence Proofs for \myalgo}

\begin{reptheorem}{thm:crr-final}
\label{repthm:crr-final}
For $\ko \leq k \leq n/10000$ and Gaussian designs, with probability at least $1-\delta$, \myalgo, after $\bigO{\log\frac{\norm{\bo}_2}{n} + \log\frac{n}{d}}$ steps, ensures that $\norm{\wt - \wo}_2 \leq \bigO{\frac{\sigma}{\lambda_{\min}(\Sigma)}\sqrt{\frac{d}{n}\log\frac{nd}{\delta}}}$.
\end{reptheorem}
\begin{proof}
Putting Lemmata~\ref{lem:lambda-fa-link} and \ref{lem:lambda-bound} establishes that
\[
\norm{\vln}_2 \leq 0.99\norm{\vlt}_2 + C\sigma\sqrt{\frac{d}{n}\log\frac{nd}{\delta}},
\]
which ensures a linear convergence of the terms $\norm{\vlt}_2$ to a value $\bigO{\sigma\sqrt{\frac{d}{n}\log\frac{nd}{\delta}}}$. Applying Lemma~\ref{lem:w-lambda-link} then finishes off the result.
\end{proof}

\begin{replemma}{lem:coarse-conv}
\label{replem:coarse-conv}
For any data matrix $X$ that satisfies the SSC and SSS properties such that $\frac{2\Lambda_{k+\ko}}{\lambda_n} < 1$, \myalgo, when executed with a parameter $k \geq \ko$, ensures that after $T_0 = \bigO{\log\frac{\norm{\bo}_2}{\sqrt n}}$ steps, $\norm{\vb^{T_0} - \bo}_2 \leq 3e_0$, where $e_0 = \bigO{\sigma\sqrt{(k+\ko)\log\frac{n}{\delta(k+\ko)}}}$ for standard Gaussian designs.
\end{replemma}
\begin{proof}
We start with the update step in \myalgo, and use the fact that $\vy = X^\top\wo + \bo + \veps$ to rewrite the update as
\[
\btn \< \HT_{k}(P_X\bt + (I - P_X)(X^\top\wo + \bo + \veps)).
\]
Since $X^\top = P_XX^\top$, we get, using the notation set up before,
\[
\btn \< \HT_{k}(\bo + X^\top\vlt + \vg).
\]
Since $k \geq \ko$, using the properties of the hard thresholding step gives us
\[
\norm{\btn_{\Itn} - (\bo_{\Itn} + X_{\Itn}^\top\vlt  + \vg_{\Itn})}_2 \leq \norm{\bo_{\Itn} - (\bo_{\Itn} + X_{\Itn}^\top\vlt + \vg_{\Itn})}_2 = \norm{X_{\Itn}^\top\vlt + \vg_{\Itn}}_2.
\]
This, upon applying the triangle inequality, gives us
\[
\norm{\btn - \bo}_2 \leq 2\norm{X_{\Itn}^\top\vlt + \vg_{\Itn}}_2.
\]
Now, using the SSC and SSS properties of $X$, we can show that $\norm{X_{\Itn}^\top\vlt}_2 = \norm{X_{\Itn}^\top(XX^\top)^{-1}X^\top_{\It}(\bt - \bo)}_2 \leq \frac{\Lambda_{k+\ko}}{\lambda_n}\norm{\bt - \bo}_2$.

Since $\veps$ is a Gaussian vector, using tail bounds for Chi-squared random variables (for example, see \cite[Lemma 20]{BhatiaJK2015}), for any set $S$ of size $k + \ko$, we have with probability at least $1 - \delta$ ,$\norm{\veps_S}_2^2 \leq \sigma^2 (k+\ko) + 2e\sigma^2\sqrt{6(k+\ko)\log\frac{1}{\delta}}$. Taking a union bound over all sets of size $(k+\ko)$ and $\binom{n}{k} \leq \br{\frac{en}{k}}^k$ gives us, with probability at least $1 - \delta$, for all sets $S$ of size at most $(k+\ko)$,
\[
\norm{\veps_S}_2 \leq \sigma\sqrt{(k+\ko)}\sqrt{1 + 2e\sqrt{6\log\frac{en}{\delta(k+\ko)}}}
\]


Using tail bounds on Gaussian random variables\footnote{$\displaystyle \frac{1}{\sqrt{2\pi}}\int_x^\infty e^{-t^2/2}dt \leq \frac{1}{\sqrt{2\pi}}\int_x^\infty\frac{t}{x}e^{-t^2/2}dt = \frac{1}{x\sqrt{2\pi}}e^{-x^2/2}$}, we can also show that for every $i$, with probability at least $1 - \delta$, we have $\norm{(X\epsilon)_i}_2 \leq \sigma\norm{(X^\top)_i}_2\sqrt{2\log\frac{1}{\delta}}$. Taking a union bound gives us, with the same confidence, $\norm{X\veps}_2^2 \leq 2\sigma^2\norm{X}_F^2\log\frac{d}{\delta} \leq 2\sigma^2d\Lambda_n\log\frac{d}{\delta}$. This allows us to bound $\norm{\vg_{\Itn}}_2$
\begin{align*}
\norm{\vg_{\Itn}}_2 &= \norm{\veps_{\Itn} - X_{\Itn}^\top(XX^\top)^{-1}X\veps}_2\\
											&\leq \sigma\sqrt{(k+\ko)}\sqrt{1 + 2e\sqrt{6\log\frac{en}{\delta(k+\ko)}}} + \sigma\frac{\sqrt{\Lambda_{k+\ko}\Lambda_n}}{\lambda_n}\sqrt{2d\log\frac{d}{\delta}}\\
											&\leq \underbrace{\sigma\sqrt{(k+\ko)}\sqrt{1 + 2e\sqrt{6\log\frac{en}{\delta(k+\ko)}}}}_{e_0}\br{1 + \sqrt{\frac{2d}{n}\log\frac{d}{\delta}}}\\
											&= 1.0003e_0,
\end{align*}
where the second last step is true for Gaussian designs and sufficiently large enough $n$. Note that $e_0$ does note depend on the iterates and is thus, a constant. This gives us
\[
\norm{\btn - \bo}_2 \leq \frac{2\Lambda_{k+\ko}}{\lambda_n}\norm{\bt - \bo}_2 + 2.0006e_0.
\]
For data matrices sampled from Gaussian ensembles, whose SSC and SSS properties will be established later, assuming $n \geq d \log d$, we have $e_0 = \bigO{\sigma\sqrt{(k+\ko)\log\frac{n}{\delta(k+\ko)}}}$. Thus, if $\frac{2\Lambda_{k+\ko}}{\lambda_n} < 1$, then in $T_0 = \bigO{\log\frac{\norm{\bo}_2}{e_0}} = \bigO{\log\frac{\norm{\bo}_2}{\sqrt n}}$ steps, \myalgo ensures that $\norm{\vb^{T_0} - \bo}_2 \leq 2.0009e_0$.
\end{proof}

\begin{lemma}
\label{lem:w-lambda-link}
Let $\lambda_{\min}(\Sigma)$ be the smallest eigenvalue of the covariance matrix of the distribution $\cN(\vzero,\Sigma)$ that generates the data points. Then at any time instant $t$, we have  $\norm{\wt - \wo}_2 \leq \frac{2}{\lambda_{\min}(\Sigma)}\br{2\sigma\sqrt{\frac{d}{n}\log\frac{d}{\delta}} + \norm{\vlt}_2}$.
\end{lemma}
\begin{proof}
As described in Algorithm~\ref{algo:myalgo}, $\wt = (XX^\top)^{-1}X(\vy - \bt) = \wo + (XX^\top)^{-1}X(\veps + \bo - \bt)$. Thus, we get
\begin{align*}
\norm{\wt - \wo}_2 &\leq \frac{1}{\lambda_{\min}(XX^\top)} \norm{X^\top(\wt - \wo)}_2\\
									 &\leq \frac{1}{n\lambda_{\min}(\Sigma) - C_\Sigma\sqrt n} \norm{X^\top(\wt - \wo)}_2\\
									 &\leq \frac{1}{n\lambda_{\min}(\Sigma) - C_\Sigma\sqrt n} \norm{\bar X^\top(\bar X\bar X^\top)^{-1}\bar X(\veps + \bo - \bt)}_2\\
									 &\leq \frac{\Lambda_n}{n\lambda_{\min}(\Sigma) - C_\Sigma\sqrt n} \norm{(\bar X\bar X^\top)^{-1}\bar X(\veps + \bo - \bt)}_2\\
									 &\leq \frac{2}{\lambda_{\min}(\Sigma)}\br{2\sigma\sqrt{\frac{d}{n}\log\frac{d}{\delta}} + \norm{\vlt}_2},
\end{align*}
where the second step follows from results on eigenvalue bounds for data matrices drawn from non-spherical Gaussians, where $C_\Sigma$ is a constant dependent on the subGaussian norm of the distribution, and the last step assumes $n \geq \frac{2C_\Sigma}{\lambda_{\min}(\Sigma)}$ and uses the proof technique used in Lemma~\ref{lem:coarse-conv} to get
\[
\norm{(\bar X\bar X^\top)^{-1}\bar X\veps}_2 \leq \sigma\frac{\sqrt{\Lambda_n}}{\lambda_n}\sqrt{2d\log\frac{d}{\delta}} \leq 2\sigma\sqrt{\frac{d}{n}\log\frac{d}{\delta}}.
\]
\end{proof}

\begin{replemma}{lem:lambda-fa-link}
\label{replem:lambda-fa-link}
Suppose $\ko \leq k \leq n/10000$. Then with probability $1-\delta$, at every time instant $t > T_0$, \myalgo ensures that $\norm{\vln}_2 \leq \frac{1}{100}\norm{\vlt}_2 + 2\sigma\sqrt{\frac{2d}{n}\log\frac{d}{\delta}} + \frac{2.001}{\lambda_n}\norm{X_\fan(X_\fan^\top\vlt + \vg_\fan)}_2$.
\end{replemma}
\begin{proof}
We have $\btn = \HT_k(\bo + X^\top\vlt + \vg)$. To analyze $\vln = (XX^\top)^{-1}X(\btn - \bo)$, we start by looking at $X(\btn - \bo) = X_\mdn(\btn_\mdn - \bo_\mdn) + X_\fan(\btn_\fan - \bo_\fan) + X_\cin(\btn_\cin - \bo_\cin)$. We then have
\begin{align*}
X_\mdn(\btn_\mdn - \bo_\mdn) &= X_\mdn(-\bo_\mdn)\\
X_\cin(\btn_\cin - \bo_\cin) &= X_\cin(X_\cin^\top\vlt + \vg_\cin)\\
X_\fan(\btn_\fan - \bo_\fan) &= X_\fan(X_\fan^\top\vlt + \vg_\fan).
\end{align*}
This gives us upon completing the terms, and using $\cin \uplus \mdn = S^\ast$,
\[
X(\btn - \bo) = X_\fan(X_\fan^\top\vlt + \vg_\fan) + X_{S^\ast}(X_{S^\ast}^\top\vlt + \vg_{S^\ast}) - X_\mdn(\bo_\mdn + X_\mdn^\top\vlt + \vg_\mdn).
\]
Now due to the hard thresholding operation, we have $\norm{\bo_\mdn + X_\mdn^\top\vlt + \vg_\mdn}_2 \leq \norm{X_\fan^\top\vlt + \vg_\fan}_2$. This gives us
\begin{align*}
\norm{X_\mdn(\bo_\mdn + X_\mdn^\top\vlt + \vg_\mdn)}_2 &= \norm{X(\bo_\mdn + X_\mdn^\top\vlt + \vg_\mdn)}_2\\
																											 &\leq \Lambda_n\norm{\bo_\mdn + X_\mdn^\top\vlt + \vg_\mdn}_2\\
																											 &\leq \Lambda_n\norm{X_\fan^\top\vlt + \vg_\fan}_2\\
																											 &\leq \frac{\Lambda_n}{\lambda_n}\norm{X(X_\fan^\top\vlt + \vg_\fan)}_2\\
																											 &= \frac{\Lambda_n}{\lambda_n}\norm{X_\fan(X_\fan^\top\vlt + \vg_\fan)}_2\\
																											 &\leq 1.001\norm{X_\fan(X_\fan^\top\vlt + \vg_\fan)}_2,
\end{align*}
where the last step uses a large enough $n$ so that the data matrix $X$ is well conditioned. Thus,
\begin{align*}
\norm{\vln}_2 &= \norm{(XX^\top)^{-1}X(\btn - \bo)}_2\\
							&\leq \frac{1}{\lambda_n}\norm{X_{S^\ast}(X_{S^\ast}^\top\vlt + \vg_{S^\ast})}_2 + \frac{2.001}{\lambda_n}\norm{X_\fan(X_\fan^\top\vlt + \vg_\fan)}_2\\
							&\leq \frac{1}{100}\norm{\vlt}_2 + 2\sigma\sqrt{\frac{2d}{n}\log\frac{d}{\delta}} + \frac{2.001}{\lambda_n}\norm{X_\fan(X_\fan^\top\vlt + \vg_\fan)}_2,
\end{align*}
where the third step follows by observing that the columns of $X$ are (statistically equivalent to) i.i.d. samples from a standard Gaussian, the fact that the support of the corruptions $S^\ast$ is chosen independently of the data and the noise, and requiring that $\ko \leq \frac{n}{100}$.
\end{proof}

\begin{replemma}{lem:lambda-bound}
\label{replem:lambda-bound}
Suppose $\ko \leq k \leq n/10000$. Then with probability at least $1 - \delta$, \myalgo ensures at every time instant $t > T_0$, for some constant $C$
\[
\frac{2.001}{\lambda_n}\norm{X_\fan(X_\fan^\top\vlt + \vg_\fan)}_2 \leq 0.98\norm{\vlt}_2 + C\sigma\sqrt{\frac{d}{n}\log\frac{nd}{\delta}}
\]
\end{replemma}
\begin{proof}
For this we first observe that, since entries in the set $\fan$ survived the hard thresholding step, they must have been the largest elements by magnitude in the set $\bar{S^\ast}$ i.e.
\[
X_\fan^\top\vlt + \vg_\fan = \HT_{\abs{\fan}}(X_{\bar{S^\ast}}^\top\vlt + \vg_{\bar{S^\ast}})
\]
Note that $\abs{\fan} \leq k$ and $\bar{S^\ast}$ is a fixed set of size $n-\ko$ with respect to the data points and the Gaussian noise. Thus, if we denote by $S^t_k$, the set of top $k$ coordinates by magnitude in $\bar{S^\ast}$ i.e.
\[
X_{S^t_k}^\top + \vg_{S^t_k} = \HT_{k}(X_{\bar{S^\ast}}^\top\vlt + \vg_{\bar{S^\ast}}),
\]
then $\norm{X_\fan(X_\fan^\top\vlt + \vg_\fan)}_2 \leq \norm{X_{S^t_k}(X_{S^t_k}^\top\vlt + \vg_{S^t_k})}_2$. Thus, all we need to do is bound this term. In the following, we will, for sake of simplicity, omit the subscript $\bar{S^\ast}$.

Before we move ahead, we make a small change to notation for convenience. At the moment, we are defining $\vlt = (XX^\top)^{-1}(\bt - \bo)$ and $\vg = (I - X^\top(XX^\top)^{-1}X)\veps$ and analyzing the vector $X^\top\vlt + \vg$. However, this is a bit cumbersome since $\vg$ is not distributed as a spherical Gaussian, something we would like to be able to use in the subsequent proofs. To remedy this, we simply change notation to denote $\vlt = (XX^\top)^{-1}(\bt - \bo) - (XX^\top)^{-1}X\veps$ and $\vg = \veps$. This will not affect the results in the least since we have, as shown in the proof of Lemma~\ref{lem:coarse-conv}, $\norm{(XX^\top)^{-1}X\veps}_2 \leq \sigma\sqrt{\frac{2d}{n}\log\frac{d}{\sigma}}$ because of which we can set $n$ large enough so that $\norm{\vlt}_2 \leq \frac{\sigma}{100}$ still holds. Given this, we prove the following result:

\begin{lemma}
\label{lem:lambda-bound-prelim}
Let $X = [\vx_1,\vx_2,\ldots,\vx_n]$ be a data matrix consisting of i.i.d. standard normal vectors i.e $\vx_i \sim \cN(\vzero,I_{d \times d})$, and $\vg \sim N(0,\sigma^2\cdot I_{n \times n})$ be standard normal vector drawn independently of $X$. For any $\vla \in \bR^d$ such that $\norm{\vla}_2 \leq \frac{\sigma}{100}$, define $\vv = X^\top\vla + \vg$. For any $\tau > 0$, define the vector $\vz$ such that $z_i = v_i$ if $\abs{v_i} > \tau$ and $z_i = 0$ otherwise. Then, with probability at least $1-\delta$, for all $\vla \in \bR^d$ with norm at most $\frac{\sigma}{100}$, we have $\frac{1}{\lambda_n}\norm{X\vz}_2 \leq M(\tau) \norm{\vla}_2 + 2.02\sigma \sqrt{\frac{d}{n}\log\frac{nd}{\delta}}$, where $M(\tau) < \frac{0.808}{\sigma}\br{\tau + \frac{1}{\tau}}\exp\br{-\frac{\tau^2}{2.001\sigma^2}}$.
\end{lemma}
\begin{proof}
We will first prove this result by first assuming that $\vla$ is a fixed $d$-dimensional vector with small norm and $X$ and $\veps$ are chosen independently of $\vla$. We will then generalize to all small norm vectors in $\bR^d$ by taking a suitably fine $\epsilon$-net over them. Let us denote the $i\nth$ row of $X$ as $X^i$, and the entry at the $j\nth$ column in this row as $X^i_j$. Then $(X\vz)_i = \vz^\top X^i = \sum_{j=1}^n{X^i_jz_j}$. Note that $v_j = \vx_j^\top\vla + g_j$ and hence $v_j$ and $v_{j'}$ are independent for $j \neq j'$. Because of this, $X^i_jz_j$ is also independent from $X^i_{j'}z_{j'}$.

We also note that $v_j|X^i_j \sim \cN(X^i_j\lambda_i,\sigma^2 + \sum_{i' \neq i}\lambda_{i'}^2)$. Let $\tilde\sigma^2 := \sigma^2 + \sum_{i' \neq i}\lambda_{i'}^2$. Note that $z_i = \ind{\abs{v_i} > \tau}\cdot v_i$. Using a simpler notation temporarily $x := X^i_j, z := z_j$ and $v := v_j$ lets us write
\[
\E{xz} = \int_{\bR\backslash[-\tau,\tau]}\int_{\bR}xv\ p(x,v)\ dx\ dv.
\]
Let $D_i := \br{I + \frac{\lambda_i^2}{\sigma^2}}^{1/2}$. Then for any fixed $v$, we have
\begin{align*}
\int_{\bR}xv\ p(x,v)\ dx &= \int_{\bR}xv\ p(x)p(v|x)\ dx\\
																			&= \frac{1}{\tilde\sigma(\sqrt{2\pi})^2}\int_{\bR}xv\ \exp\br{-\frac{x^2}{2}} \exp\br{-\frac{(v - x\lambda_i)^2}{2\tilde\sigma^2}}\ dx\\
																			&= \frac{vD_i^{-2}}{\tilde\sigma(\sqrt{2\pi})^2}\int_{\bR}u\ \exp\br{-\frac{u^2}{2} + \frac{v^2}{\tilde\sigma^2} - \frac{2vuD_i^{-1}\lambda_i}{\tilde\sigma^2}} \ du\\
																			&= \frac{vD_i^{-2}\exp\br{-\frac{v^2}{2\tilde\sigma^2} + \frac{v^2D_i^{-2}\lambda_i^2}{2\tilde\sigma^4}}}{\tilde\sigma(\sqrt{2\pi})^{2}}\int_{\bR}u\ \exp\br{-\frac{1}{2}\br{u - \frac{vD_i^{-1}\lambda_i}{\tilde\sigma^2}}^2} \ du\\
																			&= \frac{v^2D_i^{-3}\lambda_i\exp\br{-\frac{v^2}{2\tilde\sigma^2} + \frac{v^2D_i^{-2}\lambda_i^2}{2\tilde\sigma^4}}}{\tilde\sigma^3\sqrt{2\pi}}\\
																			&\leq \frac{v^2D_i^{-3}\lambda_i\exp\br{-\frac{v^2}{2.001\sigma^2}}}{1.001\sigma^3\sqrt{2\pi}},
\end{align*}
where in the third step, we perform a change of variables $u = D_ix$ and in the last step, we use the fact that $\tilde\sigma^2 \leq \sigma^2 + \sigma^2/10000$ since $\norm{\vla}_2 \leq \sigma/100$, as well as $\lambda_i^2 \leq \norm{\vla}_2^2$. Plugging this into the expression for $\E{xz}$ and using elementary manipulations such as integration by parts gives us
\[
\E{X^i_jz_j} = M(\tau) \lambda_i,\  i.e., \ \ \E{\vla^T\vx_jz_j} = M(\tau) \|\vla\|^2_2,
\]
where $M(\tau) < 0.8\br{\frac{\tau}{\sigma} + \frac{\sigma}{\tau}}\exp\br{-\frac{\tau^2}{2.001\sigma^2}}$. This gives us $\E{\sum_{j=1}^n\vla^T\vx_jz_j} = nM(\tau) \|\vla\|_2^2$. Moreover, for any $j$, $\vla^T\vx_j$ is a $\|\vla\|_2$-subGaussian random variable and $z_j$ is a $2\sigma$-subGaussian random variable as $\|\vla\|_2\leq \sigma/100$. Hence, $\vla^T\vx_jz_j$ is a sub-exponential random variable with sub-exponential norm $2 \sigma\|\vla\|_2$. 
Using the Bernstein inequality for subexponential variables \cite{Vershynin2012}, then allows us to arrive at the following result, with probability at least $1-\delta$.
\[
\sum_{j=1}^n\vla^T\vx_jz_j \leq nM(\tau) \|\vla\|_2^2 + 2\sqrt{\sigma\|\vla\|_2}\sqrt{n\log\frac{2}{\delta}}.
\]
Taking a union bound over an $\epsilon$-net over all possible values of $\vla$ (i.e. which satisfy the norm bound), for $\epsilon = 1/100$ gives us, with probability at least $1-\delta$, for all $\vla \in \bR^d$ satisfying $\norm{\vla}_2 \leq \frac{\sigma}{100}$,
\begin{equation}\label{eq:lxz}
\frac{1}{\lambda_n}\vla^TX\vz \leq 1.01 M(\tau) \norm{\vla}_2^2 + 2.02 \sqrt{\sigma\|\vla\|_2}\sqrt{\frac{d}{n}\log\frac{200}{\delta}}.
\end{equation}
Now, again consider a fixed $\vla$ and a fixed {\em unit} vector $\vv\in \mathbb{R}^d$ s.t. $\vla^T \vv=0$. In this case, $\vv^T\vx_j$ is independent of $z_j$. Hence, $\E[\vv^T\vx_j z_j]=0$. Moreover, $\vv^T\vx_j z_j$ is a $2\sigma$-subexponential random variable. Moreover, number of fixed $\vla$ and $\vv$ in their $\epsilon$-net is $\frac{1}{\epsilon}^d\cdot \frac{1}{\epsilon}^{d-1}$. Hence, using the subexponential Bernstein inequality and using union bound over all $\vv$ and $\vla$, we get (w.p. $\geq 1-\delta$): 
\begin{equation}\label{eq:vxz}
\max_{\vv, \vla}\frac{1}{\lambda_n}\vv^TX\vz \leq 2.02 \sqrt{\sigma\frac{d}{n}\log\frac{200}{\delta}}.
\end{equation}
Lemma now follows by using $\|X\vz\|_2^2=\frac{1}{\|\lambda\|_2^2} (\vla^TX\vz)^2 + \max_{\vv, \|\vv\|_2=1, \vv^T\vla=0}(\vv^TX\vz)^2$ with \eqref{eq:lxz} and \eqref{eq:vxz}. 

This establishes the claimed result.
\end{proof}

Although Lemma~\ref{lem:lambda-bound-prelim} seems to close the issue of convergence of the iterates $\vlt$, and hence the convergence of $\wt$ and consistency, it is not so. The reason is twofold -- firstly Lemma~\ref{lem:lambda-bound-prelim} works with a value based thresholding whereas \myalgo uses a cardinality based thresholding. Secondly, in order to establish a linear convergence rate for $\vlt$, we need to show that the constant $M(\tau)$ is smaller than $98/100$ so that we can ensure that $\norm{\vln}_2 \leq \br{\frac{1}{100} + 0.98}\norm{\vlt}_2 \leq 0.99\norm{\vlt}_2 + \softO{\sqrt\frac{d}{n}}$, thus ensuring a linear convergence for $\vlt$, save negligible terms. We do both of these in the subsequent discussion.

We address both the above issues by showing that while thresholding the vector $X^\top\vlt + \vg$ (recall that for sake of notational convenience we are still omitting the subscript $\bar{S^\ast})$, the $k\nth$ top element in terms of magnitude will be large enough. Thus, thresholding at that value will recover the top $k$ elements. If we are able to get a sample independent bound on the magnitude of this element then we can set $\tau$ to this in the analysis of Lemma~\ref{lem:lambda-bound-prelim} and be done. Of course, it will still have to be ensured that for this value of $\tau$, we have $M(\tau) < 1$.

To simplify the discussion and calculations henceforth, we shall assume that $\sigma = 1$, $\delta = 1$, and $k = \ko$. We stress that all our analyses go through even for non-unit variance noise, projection parameters that differ from the true corruption sparsity (i.e. $k \neq \ko$), as well as can be readily modified to give high confidence bound. However, these assumptions greatly simplify our analyses.

We notice that the vector being thresholded has two components $X^\top\vlt$ and $\vg$. Whereas $\vg$ has a nice characterization, being a standard Gaussian vector, there is very little we can say about the vector $X^\top\vlt$ other than that the norm of the vector $\vlt$ is small. This is because the vector $\vlt$ is dependent on previous iterations and hence, dependent on $X$ as well as $\vg$. The way out of this is to show that the $k\nth$ largest element in $\vg$ is reasonably large and $X^\top\vlt$, on account of its small norm, cannot diminish it.

To proceed in this direction, we first recall the coarse convergence analysis. Letting $\alpha := \frac{\ko}{n}$ and making the assumptions stated above we know that $\|\vla^{T_0}\|_2 \leq \cC(\alpha)$ where
\[
\cC(\alpha) = 2.001\sqrt{2\alpha}\sqrt{1 + 2e\sqrt{6\log\frac{e}{2\alpha}}}.
\]
Note that ${\lim}_{\alpha \rightarrow 0}\ \cC(\alpha) = 0$, as well as that $\norm{X^\top\vlambda^{T_0}}_2 \leq \cC(\alpha)\cdot\sqrt n$. This bound gives us an idea about how much weight lies in the vector $X^\top\vlt$ in the iterations $t > T_0$. Next we look at the other component $\vg$. For any value $\eta > 0$, the probability of a Gaussian variable exceeding that value in magnitude is given by $\sqrt 2\cdot\erfc(\eta/\sqrt 2)$, where $\erfc$ is the complimentary error function. By an application of Chernoff bounds, we can then conclude that in any ensemble of $n$ such Gaussian variables, with probability at least $1-\exp(-\Omega(n))$ at least a $0.99\cdot\erfc\br{\frac{\eta}{\sqrt 2}}$ fraction (as well as at most a $1.01\cdot\erfc\br{\frac{\eta}{\sqrt 2}}$ fraction) of points will exceed the value $\eta$.

We also recall the quantity
\[
M(\zeta) < 0.8\br{\zeta + \frac{1}{\zeta}}\exp\br{-\frac{\zeta^2}{2.001}},
\]
and notice that, in order for $M(\zeta)$ to get less than $98/100$, $\zeta$ must be greater than 0.99. Now the previous estimate for bounds on Gaussian variables tells us that with probability at least $1-\exp(-\Omega(n))$, at least a $\beta = 1/25$ fraction of values in the vector $\vg$, which is a standard Gaussian (since we have assumed $\sigma = 1$ for sake of simplicity) will exceed the value 1.98.

Let $S_\beta$ denote the set of coordinates of $\vg$ which exceed the value 1.98. Let us call a coordinate $i \in S_\beta$ \emph{corrupted} if $\abs{(X^\top\vlambda^{T_0})_i} \geq 0.98$. Now we notice that if this happens for $(\beta - \alpha)\cdot n$ points in the set $S_\beta$, then $\norm{X^\top\vlambda^{T_0}}_2 \geq 0.98\sqrt{(\beta - \alpha)n}$. Thus, we set $\cC(\alpha)\cdot\sqrt n < 0.98\sqrt{(\beta - \alpha)n} = 0.98\sqrt{(0.04 - \alpha)n}$ to prevent this from happening. We note that for all values of $\alpha < \frac{1}{10000}$ this is true. This ensures that at least $\ko = \alpha\cdot n$ points in the set $S$ are of magnitude at least $1$ and thus we can set $\tau = 1$ in Lemma~\ref{lem:lambda-bound-prelim} which then finishes the proof since $M(1) < 0.98$.
\end{proof}
\section{Supplementary Material for Consistent Robust Time Series Estimation}
\label{app:app_add}
\subsection{Main Result}

\begin{reptheorem}{thm:crtse-final}
	Let $\vy$ be generated using $\ARD$ process with $\ko$ additive outliers (see \eqref{ao-corruption-model-specification}). Also, let $\ko \leq k \leq C \frac{\mathfrak{m}_{\wo}}{\cM_{\wo} + \cM_{W}}\frac{n}{d \log{n}}$ (for some universal constant $C > 0$). Then, with probability at least $1-\delta$, \aoard, after $\cO(\log(\norm{\bo}_2/n) + \log(n/(\sigma\cdot d)))$ steps, ensures that $\norm{\wt - \wo}_2 \leq \bigO{\sigma \cM_{\wo}/\mathfrak{m}_{\wo} \sqrt{d\log{n}/n \log\br{d/ \delta}}}$.
\end{reptheorem}
\begin{proof}
	Putting together the Lemma~\ref{lem:lambda-bound-ts} and the equation \eqref{ard-recurrence-lemma-eq} establishes that
	\[
	\norm{\vln}_2 \leq 0.51 \norm{\vlt}_2 + \bigO{\sigma\sqrt{\frac{d\log n}{n}\log\frac{d}{\delta}}},
	\]
	which ensures a linear convergence of the terms $\norm{\vlt}_2$ to a value $\bigO{\sigma\sqrt{\frac{d\log n}{n}\log\frac{d}{\delta}}}$. Applying the equation \eqref{ard-consistency-lemma-eq}, then finishes off the result.
\end{proof}

\subsection{Back ground on Time Series}
\label{time-series-notes}

$\ARD$ process is defined as
\begin{equation}
\label{clean-ar-d-model}
x_t ~=~ x_{t-1} \wo_1 + \cdots + x_{t-d} \wo_d + \veps_t  \text{ where } \veps_t \sim \cN(0,\sigma^2).  
\end{equation}
Note that $x_t \sim \cN(0,\Gamma(0))$, where $\Gamma\br{h}=\E{x_t x_{t+h}}$ is the auto-covariance function of the time series. Then we have
\begin{align}
\begin{bmatrix}
x_{1} \\
\vdots \\
x_{n}
\end{bmatrix}
~=~&
\begin{bmatrix}
x_0 & \cdots & x_{-d+1} \\
\vdots & & \vdots \\
x_{n-1} & \cdots & x_{n-d}
\end{bmatrix}
\cdot
\begin{bmatrix}
\wo_1 \\
\vdots \\
\wo_d
\end{bmatrix}
+
\begin{bmatrix}
\veps_1 \\
\vdots \\
\veps_n
\end{bmatrix} \nonumber
\\
\yo ~=~& \bar{X}^\top \wo + \veps . \label{clean-ar-process-model}
\end{align}
The spectral density of this $\ARD$ process can be given as
\begin{equation}
\label{spectral}
\rho_{\wo}\br{\omega} = \frac{\sigma^2}{\br{1-\sum_{k=1}^{d}{\wo_k e^{ik\omega}}} \br{1-\sum_{k=1}^{d}{\wo_k e^{-ik\omega}}}}, \text{ for } \omega \in \bs{0,2\pi}.
\end{equation}

Observe that any column vector of the matrix $\bar{X}$ is distributed as $\bar{X}_i \sim \cN\br{0,C_{\bar{X}}}$, where 
\[
C_{\bar{X}} ~=~ 
\begin{bmatrix}
\Gamma\br{0} & \Gamma\br{1} & \cdots & \Gamma\br{d-1} \\
\Gamma\br{1} & \Gamma\br{0} & \cdots & \Gamma\br{d-2} \\
\vdots & \vdots & \ddots & \vdots \\
\Gamma\br{d-1} & \Gamma\br{d-2} & \cdots & \Gamma\br{0}
\end{bmatrix}.
\]

Since $C_{\bar{X}}$ is a block-Toeplitz matrix, we have
\begin{equation}
\label{spectral-bounds}
\mathfrak{m}_{\wo} := \inf_{\omega \in \bs{0,2\pi}} \rho_{\wo}\br{\omega} \leq \Lambda_{\min} \bs{C_{\bar{X}}} \leq \Lambda_{\max} \bs{C_{\bar{X}}} \leq \sup_{\omega \in \bs{0,2\pi}} \rho_{\wo}\br{\omega} =: \cM_{\wo} .
\end{equation}

The columns of $\bar{X}$ can be viewed as a $d$-variate of $\VAR$ process as follows 
\begin{align}
\begin{bmatrix}
x_{i} \\
x_{i-1} \\
\vdots \\
x_{i-(d-1)}
\end{bmatrix}
~=~&
\begin{bmatrix}
\wo_1 & \wo_2 & \cdots & \wo_{d-1} & \wo_d \\
1 & 0 & \cdots &0 & 0 \\
0 & 1 & \cdots &0 & 0 \\
\vdots & \vdots & \ddots &\vdots & \vdots \\
0 & 0 & \cdots & 1 & 0
\end{bmatrix} 
\cdot
\begin{bmatrix}
x_{i-1} \\
x_{i-2} \\
\vdots \\
x_{i-d}
\end{bmatrix} + 
\begin{bmatrix}
\veps_{i} \\
0 \\
\vdots \\
0
\end{bmatrix}  \nonumber
\\
\hat{X}_{i} ~=~& W \hat{X}_{i-1} + \cE_{i} , \text{ for } i=1,\ldots,n . \label{var-modify}
\end{align}
By letting 
\[
\uo =  
\begin{bmatrix}
\hat{X}_1 \\
\vdots \\
\hat{X}_{n}
\end{bmatrix} \in \bR^{nd}, ~
\cU = 
\begin{bmatrix}
\hat{X}_0 \\
\vdots \\
\hat{X}_{n-1}
\end{bmatrix} \in \bR^{nd}, \text{ and } 
\cE = 
\begin{bmatrix}
\cE_1 \\
\vdots \\
\cE_{n}
\end{bmatrix} \in \bR^{nd}
\]
the above $\VAR$ process can be compactly written as follows 
\[
\uo ~=~ W \cU + \cE .
\]
Then the spectral density of the above $\VAR$ process is given by 
\[
\rho_{W} (\omega) ~=~ \br{I-W e^{-i\omega}}^{-1} \Sigma_\veps \bs{\br{I-We^{-i\omega}}^{-1}}^*   , \text{ for } \omega \in \bs{0,2\pi},
\]
where 
\[
\Sigma_\veps ~=~ \begin{bmatrix}
\sigma^2 & 0 & \cdots & 0 \\
0 & 0 & \cdots & 0 \\
\vdots & \vdots & \ddots & \vdots \\
0 & 0 & \cdots & 0 
\end{bmatrix}.
\]
The covariance matrix of vector $\cU$ is given by
\[
C_{\cU} ~=~ \E{\cU\cU^\top}  ~=~ 
\begin{bmatrix}
\E{\hat{X}_{0}\hat{X}_{0}^\top} & \E{\hat{X}_{0}\hat{X}_{1}^\top} & \cdots & \E{\hat{X}_{0}\hat{X}_{n-1}^\top} \\
\E{\hat{X}_{1}\hat{X}_{0}^\top} & \E{\hat{X}_{1}\hat{X}_{1}^\top} & \cdots & \E{\hat{X}_{1}\hat{X}_{n-1}^\top} \\
\vdots & \vdots & \ddots & \vdots \\
\E{\hat{X}_{n-1}\hat{X}_{0}^\top} & \E{\hat{X}_{n-1}\hat{X}_{1}^\top} & \cdots & \E{\hat{X}_{n-1}\hat{X}_{n-1}^\top}
\end{bmatrix}.
\]
Since $C_{\cU}$ is a block-Toeplitz matrix, we have
\begin{equation}
\label{var-spectral-bounds}
\Lambda_{\max} \bs{C_{\cU}} \leq \sup_{\omega \in \bs{0,2\pi}} \rho_{W}\br{\omega} = \frac{\sigma^2}{\inf_{\omega \in \bs{0,2\pi}} \Lambda_{\min}\bs{\br{I-W^\top e^{i\omega}} \br{I-We^{-i\omega}}}} =: \cM_{W} .
\end{equation}

Consider a vector $\vq=\bar{X}^\top \va \in \bR^n$ for any $\va \in S^{d-1}$. Since each element $\bar{X}_{i}^\top \va \sim \cN\br{0,\va^\top C_{\bar{X}} \va}$, it follows
that $\vq \sim \cN\br{0,Q_{\va}}$ where $Q_{\va} = \br{I_{n} \otimes \va^\top} C_{\cU} \br{I_{n} \otimes \va}$. From this we can note that
\begin{align}
\trace\br{Q_{\va}} ~=~& n \va^\top C_{\bar{X}} \va ~\leq~ n \Lambda_{\max}\bs{C_{\bar{X}}} ~\leq~ n \cM_{\wo} \label{trace-bound-Qa}\\
\norm{Q_{\va}}_2 ~\leq~& \norm{\va}_2^2 \Lambda_{\max}\bs{C_{\cU}} ~\leq~ \cM_{W}  \label{spectral-norm-bound-Qa}\\
\norm{Q_{\va}}_{\text{F}} ~=~& \sqrt{\trace\br{Q_{\va} Q_{\va}}} ~\leq~ \sqrt{\norm{Q_{\va}}_2 \trace\br{Q_{\va}}} ~\leq~ \sqrt{n \cM_{W} \cM_{\wo}}. \label{frob-norm-bound-Qa}
\end{align}

\paragraph{Additive Corruptions:}
Now consider the following additive corruption mechanism (at most $\ko$ data points):
\[
\vy_i ~=~ \yo_i + \eo_i ~=~ x_i + \eo_i \text{ for } i=1,\ldots,n.
\]
Since we observe the corrupted time series data $(y_{-d+1},\ldots,y_0,y_1,\ldots,y_n)$, we have
\begin{align}
\begin{bmatrix}
y_0 & y_{-1} & \cdots & y_{-d+1} \\
y_1 & y_{0} & \cdots & y_{-d+2} \\
y_2 & y_{1} & \cdots & y_{-d+3} \\
\vdots & \vdots &  & \vdots \\
y_{n-1} & y_{n-2} & \cdots & y_{n-d} 
\end{bmatrix}
~=~&
\begin{bmatrix}
x_0 & x_{-1} & \cdots & x_{-d+1} \\
x_1 & x_{0} & \cdots & x_{-d+2} \\
x_2 & x_{1} & \cdots & x_{-d+3} \\
\vdots & \vdots &  & \vdots \\
x_{n-1} & x_{n-2} & \cdots &  x_{n-d} 
\end{bmatrix}
+
\begin{bmatrix}
0 & 0 & \cdots & 0 \\
\eo_1 & 0 & \cdots & 0 \\
\eo_2 & \eo_1 & \cdots & 0 \\
\vdots & \vdots &  & \vdots \\
\eo_{n-1} & \eo_{n-2} & \cdots & \eo_{n-d}
\end{bmatrix} \nonumber \\
X^\top ~=~& \bar{X}^\top + E^\top \label{design-matrix-connection}
\end{align}
Thus the observed time series can be modeled as follows
\begin{align}
\vy ~=~& \yo + \eo \nonumber \\
~=~& \bar{X}^\top \wo + \veps + \eo \nonumber \\
~=~& (X^\top-E^\top) \wo + \veps + \eo \nonumber \\
~=~& X^\top \wo + \veps + \bo_{\eo,\wo} , \label{additive-corruption-ar-d-actual-model}
\end{align}
where $\eo=(\eo_1,\ldots,\eo_n)^\top$ is $\ko$-sparse, and $\bo_{\eo,\wo} = \eo - E^\top \wo$ is $\ko$-block-sparse with block size of $d+1$ (since $E \wo$ is $\ko$-block-sparse with block size of $d$).

\subsection{Singular values of $\bar{X}$} 
\begin{lemma}
	\label{full-singular-values}
	Let $\bar{X}$ be a matrix whose columns are sampled from a stationary and stable $\VAR$ process given by \eqref{var-modify} i.e. $\bar{X}_i \sim \cN\br{0,C_{\bar{X}}}$. Then for any $\epsilon > 0$, with probability at least $1-\delta$, $\bar{X}$ satisfies
	\begin{align*}
		\lambda_{\max}\br{\bar{X} \bar{X}^\top} ~\leq~& n \cM_{\wo} + (1-2\epsilon)^{-1} \bc{\sqrt{n \alpha_1(d,\delta,\epsilon) \cM_{W} \cM_{\wo}} + \alpha_1(d,\delta,\epsilon) \cM_{W}} \\
		\lambda_{\min}\br{\bar{X} \bar{X}^\top} ~\geq~& n \mathfrak{m}_{\wo} - (1-2\epsilon)^{-1} \bc{\sqrt{n \alpha_1(d,\delta,\epsilon) \cM_{W} \cM_{\wo}} + \alpha_1(d,\delta,\epsilon) \cM_{W}},
	\end{align*}
	where $\alpha_1(d,\delta,\epsilon) = c \log{\frac{2}{\delta}} + c d \log{\frac{3}{\epsilon}}$ for some universal constant $c$.
\end{lemma}
\begin{proof}
	Using the results from \cite{BasuM2015,MelnykB2016}, we first show that with high probability, 
	\[
	\norm{\bar{X} \bar{X}^\top - n C_{\bar{X}}}_2 ~\leq~ \epsilon_1
	\]
	for some $\epsilon > 0$. Doing so will automatically establish the following result
	\[
	n \Lambda_{\min}\bs{C_{\bar{X}}} - \epsilon_1 ~\leq~ \lambda_{\min}\br{\bar{X} \bar{X}^\top} ~\leq~ \lambda_{\max}\br{\bar{X} \bar{X}^\top} ~\leq~ n \Lambda_{\max}\bs{C_{\bar{X}}} + \epsilon_1 .
	\]
	Let $C^{d-1}\br{\epsilon} \subset S^{d-1}$ be an $\epsilon$-cover of $S^{d-1}$ (\cite{Vershynin2012}, see Definition 5.1). Standard constructions (\cite{Vershynin2012}, see Lemma 5.2) guarantee such a cover of size at most $\br{1+\frac{2}{\epsilon}}^d \leq \br{\frac{3}{\epsilon}}^d$. Further by Lemma 5.4 from \cite{Vershynin2012}, we have
	\[
	\norm{\bar{X} \bar{X}^\top - n C_{\bar{X}}}_2 ~\leq~ (1-2\epsilon)^{-1} \sup_{\vu \in C^{d-1}\br{\epsilon}} \abs{\vu^\top \br{\bar{X} \bar{X}^\top - n C_{\bar{X}}} \vu}.
	\]
	By following the analysis given in \cite{BasuM2015,MelnykB2016}, we can provide a high probability bound on $\abs{\vu^\top \br{\bar{X} \bar{X}^\top - n C_{\bar{X}}} \vu}$. For any $\vu \in S^{d-1}$, let $\vq = \bar{X}^\top \vu \sim \cN(0,Q_{\vu})$ where $Q_{\vu} = \br{I_{n} \otimes \vu^\top} C_{\cU} \br{I_{n} \otimes \vu}$. Note that $\vu^\top \bar{X} \bar{X}^\top \vu = \vq^\top \vq = \vz^\top Q_{\vu} \vz$, where $\vz \sim \cN(0,I_n)$. Also, $\vu^\top n C_{\bar{X}} \vu = \E{\vz^\top Q_{\vu} \vz}$. So, by the Hanson-Wright inequality of \cite{RudelsonV2013}, with $\norm{\vz_i}_{\psi_2} \leq 1$ since $\vz_i \sim \cN\br{0,1}$, we get
	\begin{align*}
		\P{\abs{\vu^\top \br{\bar{X} \bar{X}^\top - n C_{\bar{X}}} \vu} > \lambda} ~=~& \P{\abs{\vz^\top Q_{\vu} \vz - \E{\vz^\top Q_{\vu} \vz}} > \lambda} \\ ~\leq~& 2 \exp\br{-\frac{1}{c} \min\bc{\frac{\lambda^2}{\norm{Q_{\vu}}_F^2},\frac{\lambda}{\norm{Q_{\vu}}_2}}}.
	\end{align*}
	Setting $\lambda = \sqrt{\alpha_1(d,\delta,\epsilon)} \norm{Q_{\vu}}_F + \alpha_1(d,\delta,\epsilon) \norm{Q_{\vu}}_2$, and taking a union bound over all $C^{d-1}\br{\epsilon}$, we get
	\begin{align*}
		& \P{\sup_{\vu \in C^{d-1}\br{\epsilon}} \abs{\vu^\top \br{\bar{X} \bar{X}^\top - n C_{\bar{X}}} \vu} > \sqrt{\alpha_1(d,\delta,\epsilon)} \norm{Q_{\vu}}_F + \alpha_1(d,\delta,\epsilon) \norm{Q_{\vu}}_2} \\
		~\leq~& 2 \br{\frac{3}{\epsilon}}^d \exp\br{-\frac{1}{c} \min\bc{\frac{\lambda^2}{\norm{Q_{\vu}}_F^2},\frac{\lambda}{\norm{Q_{\vu}}_2}}} ~\leq~ \delta . 
	\end{align*}
	This implies that probability at least $1-\delta$, 
	\[
	\norm{\bar{X} \bar{X}^\top - n C_{\bar{X}}}_2 ~\leq~ (1-2\epsilon)^{-1} \bc{\sqrt{\alpha_1(d,\delta,\epsilon)} \norm{Q_{\vu}}_F + \alpha_1(d,\delta,\epsilon) \norm{Q_{\vu}}_2},
	\]
	which (along with the bounds given in \eqref{trace-bound-Qa},\eqref{spectral-norm-bound-Qa}, and \eqref{frob-norm-bound-Qa}) gives us the claimed bounds on the singular values of $\bar{X} \bar{X}^\top$.
\end{proof}

\subsection{Restricted Singular values of $\bar{X}$}
\begin{lemma}
	\label{clean-ssc-sss}
	Let $\bar{X}$ be a matrix whose columns are sampled from a stationary and stable $\VAR$ process given by \eqref{var-modify} i.e. $\bar{X}_i \sim \cN\br{0,C_{\bar{X}}}$. Then for any $k \leq \frac{n}{d}$, with probability at least $1-\delta$, the matrix $\bar{X}$ satisfies the SGSC and SGSS properties with constants
	\begin{align*}
		\Lambda_k ~\leq~& k \bc{d \cM_{\wo} + \sqrt{d \cM_{W} \cM_{\wo} \frac{1}{c} \log{\frac{en}{kd}}} + \cM_{W} \log{\frac{en}{kd}}} \\
		& + \cO\br{\sqrt{kd \alpha_2(d,\delta) \cM_{W} \cM_{\wo}}} + \cO\br{\alpha_2(d,\delta) \cM_{W}} \\
		\lambda_k ~\geq~& n \mathfrak{m}_{\wo} - \br{\frac{n}{d}-k} \bc{d \cM_{\wo} + \sqrt{d \cM_{W} \cM_{\wo} \frac{1}{c} \log{\frac{en}{n-kd}}} + \log{\frac{en}{n-kd}} \cM_{W}} \\
		& - \Omega \br{\br{1+\sqrt{\frac{n-kd}{n}}}\sqrt{n \alpha_2(d,\delta) \cM_{W} \cM_{\wo}}} - \Omega \br{\alpha_2(d,\delta) \cM_{W}}  ,
	\end{align*}
	where $\alpha_2(d,\delta)=\log{\frac{1}{\delta}} + d$ and $c$ is some universal constant.
\end{lemma}
\begin{proof}
	One can easily observe that considering the columns-restricted matrix $\bar{X}_S$ wouldn't impact the analysis of Lemma~\ref{full-singular-values}. Thus for any fixed $S \in \cS_k^{\cG}$, Lemma~\ref{full-singular-values} guarantees the following bound (since $\abs{S}=kd$)
	\begin{align*}
		\lambda_{\max}\br{\bar{X}_S \bar{X}_S^\top} ~\leq~& kd \cM_{\wo} + (1-2\epsilon)^{-1} \bc{\sqrt{kd \alpha_1(d,\delta,\epsilon) \cM_{W} \cM_{\wo}} + \alpha_1(d,\delta,\epsilon) \cM_{W}}.
	\end{align*}
	Taking a union bound over $\cS_k^{\cG}$ and noting that $\abs{\cS_k^{\cG}} \leq \br{\frac{en}{kd}}^{k}$, gives us with probability at least $1-\delta$ 
	\[
	\Lambda_k ~\leq~ kd \cM_{\wo} + (1-2\epsilon)^{-1} M,
	\]
	where
	\begin{align*}
		M ~=~& \sqrt{\br{\alpha_1 (d,\delta,\epsilon) + c k \log{\frac{en}{kd}}} kd \cM_{W} \cM_{\wo}} + \br{\alpha_1 (d,\delta,\epsilon) + c k \log{\frac{en}{kd}}} \cM_{W} \\
		~\leq~& \sqrt{\alpha_1 (d,\delta,\epsilon) kd \cM_{W} \cM_{\wo}} + k \sqrt{c d \log{\frac{en}{kd}} \cM_{W} \cM_{\wo}} + \br{\alpha_1 (d,\delta,\epsilon) + c k \log{\frac{en}{kd}}} \cM_{W}.
	\end{align*}
	If $c<1$ (which can be ensured by scaling), by setting $\epsilon=\frac{1}{2} (1-c)$ and noting that $\Theta\br{\frac{1}{c} \alpha_1 \br{d,\delta,\frac{1-c}{2}}} = \Theta\br{\log{\frac{1}{\delta}} + d}$, we get
	\begin{align*}
		\Lambda_k ~\leq~& k \bc{d \cM_{\wo} + \sqrt{d \cM_{W} \cM_{\wo} \frac{1}{c} \log{\frac{en}{kd}}} + \cM_{W} \log{\frac{en}{kd}}} \\
		& + \cO\br{\sqrt{kd \br{\log{\frac{1}{\delta}} + d} \cM_{W} \cM_{\wo}}} + \cO\br{\br{\log{\frac{1}{\delta}} + d} \cM_{W}}.
	\end{align*}
	For the second bound, we use the equality 
	\[
	\bar{X}_S \bar{X}_S^\top = \bar{X} \bar{X}^\top - \bar{X}_{\bar{S}} \bar{X}_{\bar{S}}^\top ,
	\]
	which provides the following bound for $\lambda_k$,
	\[
	\lambda_k ~\geq~ \lambda_{\min}\br{\bar{X} \bar{X}^\top} - \max_{T \in \cS_{\frac{n}{d}-k}^{\cG}} \lambda_{\max} \br{\bar{X}_T \bar{X}_T^\top} ~=~ \lambda_{\min}\br{\bar{X} \bar{X}^\top} - \Lambda_{\frac{n}{d}-k} .
	\] 
	Using Lemma~\ref{full-singular-values} to bound the first quantity and the first part of this theorem to bound the second quantity
	gives us, with probability at least $1-\delta$,
	\begin{align*}
		\lambda_k ~\geq~& n \mathfrak{m}_{\wo} - \br{\frac{n}{d}-k} \bc{d \cM_{\wo} + (1-2\epsilon)^{-1} \br{\sqrt{c d \log{\frac{en}{n-kd}} \cM_{W} \cM_{\wo}} + c \log{\frac{en}{n-kd}} \cM_{W}}} \\
		& - (1-2\epsilon)^{-1} \bc{\br{1+\sqrt{\frac{n-kd}{n}}}\sqrt{n \alpha_1 (d,\delta,\epsilon) \cM_{W} \cM_{\wo}} + 2 \alpha_1 (d,\delta,\epsilon) \cM_{W}}.
	\end{align*}
	By setting $\epsilon=\frac{1}{2} (1-c)$ we get the following bound
	\begin{align*}
		\lambda_k ~\geq~& n \mathfrak{m}_{\wo} - \br{\frac{n}{d}-k} \bc{d \cM_{\wo} + \sqrt{d \cM_{W} \cM_{\wo} \frac{1}{c} \log{\frac{en}{n-kd}}} + \log{\frac{en}{n-kd}} \cM_{W}} \\
		& - \Omega \br{\br{1+\sqrt{\frac{n-kd}{n}}}\sqrt{n \br{\log{\frac{1}{\delta}} + d} \cM_{W} \cM_{\wo}}} - \Omega \br{\br{\log{\frac{1}{\delta}} + d} \cM_{W}}.
	\end{align*}
\end{proof}

\begin{remark}
	\label{initial-rough-bounds}
	Note that $\cM_{\wo}, \cM_{W}$ and $\mathfrak{m}_{\wo}$ will depend only on the actual model parameter vector $\wo$ and $\sigma$ (not on the realized data). Moreover $\cM_{\wo}$ and $\cM_{W}$ are closely related. For example, for $\text{AR}(1)$ time-series with $0 < \wo_1 < 1$, we have $\cM_{\wo} = \cM_{W} = \frac{\sigma^2}{(1-\wo_1)^2}$ and $\mathfrak{m}_{\wo} = \frac{\sigma^2}{(1+\wo_1)^2}$. Then for sufficiently large enough n  so that $\sqrt n \ll n$, the restricted singular value bounds of $\bar{X}$ from Lemma~\ref{clean-ssc-sss} can be simplified as follows
	\begin{align*}
	\Lambda_k ~\leq~& \bigO{k \bc{d \cM_{\wo} + \sqrt{d \cM_{\wo} \cM_{W} \log{\frac{e n}{\delta k d}}} + \cM_{W} \log{\frac{e n}{\delta k d}}}}  \text{ and }\\
	\lambda_{\frac{n}{d}} ~\geq~& \Omega \br{n \mathfrak{m}_{\wo}}.
	\end{align*}
\end{remark}

\subsection{Restricted Singular values of $X$}

\begin{theorem}[SGSS/SGSC in $\ARD$ with AO model]
	\label{ao-ssc-sss}
	Let $X$ be the matrix given in \eqref{ao-corruption-model-specification} (additive corrupted $\ARD$ model setting). Then for any $k \leq \frac{n}{d}$ and sufficiently large enough $n$, with probability at least $1-\delta$, the matrix $X$ satisfies the SGSC and SGSS properties with constants
	\begin{align*}
		\Lambda_k ~\leq~& \bigO{k \bc{d \log{n} \cM_{\wo} + \sqrt{d \cM_{\wo} \cM_{W} \log{\frac{e n}{\delta k d}}} + \cM_{W} \log{\frac{e n}{\delta k d}}}} \\
		\lambda_{\frac{n}{d}} ~\geq~& \Omega \br{n \mathfrak{m}_{\wo}}.
	\end{align*}
\end{theorem}
\begin{proof}
	Recall that the matrix $X$ can be decomposed as follows
	\begin{align*}
		X ~=~& \bar{X} + E .
	\end{align*}
	Since for any $\vv \in S^{n-1}$, $\norm{E \vv}_2^2 = \sum_{i=1}^{d} {\ip{(E^\top)_i}{\vv}}^2 \leq \sum_{i=1}^{d} \norm{(E^\top)_i}_2^2 \norm{\vv}_2^2 \leq d \norm{\eo}_2^2$, we get $\norm{E}_2 \leq \sqrt{d} \norm{\eo}_2$. By using the inequality $\norm{X_S-\bar{X}_S}_2 \leq \norm{E_S}_2 \leq \norm{E}_2$ we get
	\[
	\Lambda_{\min}\bs{\bar{X}_S} - \norm{E}_2 ~\leq~ \Lambda_{\min}\bs{X_S} ~\leq~ \Lambda_{\max}\bs{X_S} ~\leq~ \Lambda_{\max}\bs{\bar{X}_S} + \norm{E}_2 .
	\]
	
	Since $\eo$ is $\ko$-sparse and $\eo_i \leq \hat{\sigma} = \cO(\sqrt{\log{n}}\sigma)$, we have $\norm{\eo}_2 \leq \cO(\sqrt{\ko \log{n}}\sigma)$. Thus from Lemma~\ref{clean-ssc-sss} and Remark~\ref{initial-rough-bounds}, for sufficiently large enough $n$ (with probability at least $1-\delta$) we get
	\begin{align*}
	\sqrt{\Lambda_k} ~\leq~& \bigO{\sqrt{k \bc{d \cM_{\wo} + \sqrt{d \cM_{\wo} \cM_{W} \log{\frac{e n}{\delta k d}}} + \cM_{W} \log{\frac{e n}{\delta k d}}}}} + \cO(\sqrt{\ko d \log{n}} \sigma)  \\
	~\leq~& \bigO{\sqrt{k \bc{d \log{n} \cM_{\wo} + \sqrt{d \cM_{\wo} \cM_{W} \log{\frac{e n}{\delta k d}}} + \cM_{W} \log{\frac{e n}{\delta k d}}}}} \\
	\sqrt{\lambda_{\frac{n}{d}}} ~\geq~& \Omega \br{\sqrt{n \mathfrak{m}_{\wo}}} - \Omega(\sqrt{\ko d \log{n}} \sigma) ~\geq~ \Omega \br{\sqrt{n \mathfrak{m}_{\wo}}}, 
	\end{align*}
	which completes the proof.
\end{proof}

\begin{remark}
	\label{rough-bounds}
	Using Theorem~\ref{ao-ssc-sss}, we can bound $\frac{\sqrt{\Lambda_{k+\ko}}}{\lambda_{\frac{n}{d}}}$ (which is required for the coarse convergence analysis of \aoard) as follows (with probability at least $1-\delta$, and sufficiently large enough $n$)
	\begin{align*}
		\frac{\sqrt{\Lambda_{k+\ko}}}{\lambda_{\frac{n}{d}}} ~\leq~& \bigO{\frac{1}{n \mathfrak{m}_{\wo}}\sqrt{k \bc{d \log{n} \cM_{\wo} + \sqrt{d \cM_{\wo} \cM_{W} \log{\frac{e n}{\delta k d}}} + \cM_{W} \log{\frac{e n}{\delta k d}}}}} \\ 
		~=~& \frac{f(\wo,\sigma) \sqrt{\log{n}}}{n} \sqrt{(k+\ko) \br{d+2e\sqrt{6d \log{\frac{en}{\delta (k+\ko) d}}}}}, 
	\end{align*}
	for some positive function $f(\wo,\sigma)$ (suppressing $\cM_{\wo}, \cM_{W}$ and $\mathfrak{m}_{\wo}$). 
	
	From Theorem~\ref{ao-ssc-sss}, it can also be observed that, if $k \leq C \frac{\mathfrak{m}_{\wo}}{\cM_{\wo} + \cM_{W}}\frac{n}{d \log{n}}$ (for some universal constant $C > 0$), then with probability at least $1-\delta$, we get $\frac{\Lambda_{k + \ko}}{\lambda_{\frac{n}{d}}} \leq \frac{\Lambda_{2k}}{\lambda_{\frac{n}{d}}} \leq \frac{1}{4}$.
\end{remark}

\subsection{Bound on $\norm{X \veps}_2$}

\begin{lemma}
	\label{lem:Xe-bound}
	Let $X$ be the matrix given in \eqref{ao-corruption-model-specification} (additive corrupted $\ARD$ model setting). Then with probability at least $1-\delta$, 
	\[
	\norm{X \veps}_2 \leq 2 \sigma \sqrt{n \sqrt{\log{n}} c' d \log{\frac{2d}{\delta}}}.
	\]
	for some constant $c' > 0$.
\end{lemma}
\begin{proof}
	We first bound the absolute value of $(\bar{X} \veps)_i = \sum_{j=1}^{n}{\veps_j x_{j-i}}$ for $i=1,\ldots,d$. Let $z_j := \veps_j x_{j-i}$. Since $\E{z_j | \veps_1,\ldots,\veps_{j-1}} = 0$, $\bc{z_j:j \in [n]}$ is a martingale difference sequence w.r.t $\bc{\veps_j:j \in [n]}$.
	Also note that for any $j$, $(z_j | \veps_1,\ldots,\veps_{j-1}) \sim \cN (0,x_{j-i}^2 \sigma^2)$. Then using the tail bounds on Gaussian random variables we have
	\[
	\P{\abs{z_j} > t | \veps_1,\ldots,\veps_{j-1}} \leq \sqrt{\frac{2}{\pi}} \frac{1}{\abs{x_{j-i}} \sigma} \exp \br{\frac{-t^2}{2 x_{j-i}^2 \sigma^2}} \leq \sqrt{\frac{2}{\pi}} \frac{1}{c \sqrt{\log{n}} \sigma^2} \exp \br{\frac{-t^2}{2 c^2 \log{n} \sigma^4}},
	\]
	since $\sup_{i \in [n]}{\abs{x_i}} \leq \bigO{\sqrt{\log{n}} \sigma}$ with high probability. Then by using Theorem~2 from \cite{Shamir2011}, we get
	\[
	\P{\sum_{j=1}^{n}{z_j} > n \epsilon} \leq \exp \br{  \frac{- \frac{1}{2 c^2 \log{n} \sigma^4} n \epsilon^2}{28 \sqrt{\frac{2}{\pi}} \frac{1}{c \sqrt{\log{n}} \sigma^2}}} = \exp \br{\frac{-n \epsilon^2}{c' \sqrt{\log{n}} \sigma^2}},
	\]
	for some constant $c'$. Similarly we also have
	\[
	\P{\sum_{j=1}^{n}{z_j} < - n \epsilon} \leq \exp \br{\frac{-n \epsilon^2}{c' \sqrt{\log{n}} \sigma^2}}.
	\] 
	Then by using the union bound we get (for any $\delta > 0$)
	\[
	\P{\abs{\sum_{j=1}^{n}{z_j}} > n \epsilon} \leq 2 \exp \br{\frac{-n \epsilon^2}{c' \sqrt{\log{n}} \sigma^2}} = \delta .
	\]
	That is with probability at least $1-\delta$ we have 
	\[
	\abs{(\bar{X} \veps)_i} \leq n \epsilon \leq \sigma \sqrt{n \sqrt{\log{n}} c' \log{\frac{2}{\delta}}}.
	\]
	Taking a union bound gives us, with the same confidence,
	\[
	\norm{\bar{X} \veps}_2^2 \leq \sigma^2 n \sqrt{\log{n}} c' d \log{\frac{2d}{\delta}}.
	\]
	
	Now we bound the absolute value of $(E \veps)_i = \sum_{j=1}^{n}{\veps_j E_{i,j}}$ for $i=1,\ldots,d$. Let $z_j := \veps_j E_{i,j}$. Note that for any $j$, $(z_j | \veps_1,\ldots,\veps_{j-1}) \sim \cN (0,E_{i,j}^2 \sigma^2)$ and $\sup_{i,j \in [n]}{\abs{E_{i,j}}} \leq \hat{\sigma} = \bigO{\sqrt{\log{n}} \sigma}$. Then by following the similar analysis as above, we have with probability at least $1-\delta$,
	\[
	\norm{E \veps}_2^2 \leq \sigma^2 n \sqrt{\log{n}} c' d \log{\frac{2d}{\delta}}.
	\]
	Then using the triangular inequality, with probability at least $1-\delta$,
	\[
	\norm{X \veps}_2 \leq \norm{\bar{X} \veps}_2 + \norm{E \veps}_2 \leq 2 \sigma \sqrt{n \sqrt{\log{n}} c' d \log{\frac{2d}{\delta}}}.
	\]
\end{proof}

\subsection{Coarse Convergence Analysis}
\begin{reptheorem}{ao-ard-coarse-convergence}
	For any data matrix $X$ that satisfies the SGSC and SGSS properties such that $\frac{4 \Lambda_{k+\ko}}{\lambda_{\frac{n}{d}}} < 1$, \aoard, when executed with a parameter $k \geq \ko$, ensures that after $T_0 = \bigO{\log\frac{\norm{\bo}_2}{\sqrt n}}$ steps, $\norm{\vb^{T_0} - \bo}_2 \leq 5e_0$, where $e_0 = \cO\br{\sigma\sqrt{(k+\ko)d\log\frac{n}{\delta(k+\ko)d}}}$ for standard Gaussian $\ARD$ process. If $k$ is sufficiently small i.e. $\ko \leq k \leq C \frac{\mathfrak{m}_{\wo}}{\cM_{\wo} + \cM_{W}}\frac{n}{d \log{n}}$ (for some universal constant $C > 0$) and $n$ is sufficiently large enough, then with probability at least $1-\delta$, we have $\frac{4 \Lambda_{k+\ko}}{\lambda_{\frac{n}{d}}} < 1$.
\end{reptheorem}
\begin{proof}
	We start with the update step in \aoard, and use the fact that $\vy = X^\top \wo + \veps + \bo$ to rewrite the update as
	\[
	\btn \< \HT_k^{\cG} (P_{X}\bt + (I - P_{X})(X^\top \wo + \veps + \bo)),
	\]
	where $P_{X} = X^\top (XX^\top)^{-1} X$. Since $X^\top = P_{X}X^\top$, we get
	\[
	\btn \< \HT_k^{\cG} (\bo + P_{X}(\bt - \bo) + (I - P_{X})\veps).
	\]
	Let $\It := \supp(\bt) \cup \supp(\bo)$, $\vlt := (XX^\top)^{-1}X(\bt - \bo)$, and $\vg := (I - P_{X})\veps$. Since $k \geq \ko$, using the properties of the hard thresholding step gives us
	\[
	\norm{\btn_{\Itn} - (\bo_{\Itn} + X_{\Itn}^\top\vlt  + \vg_{\Itn})}_2 \leq \norm{\bo_{\Itn} - (\bo_{\Itn} + X_{\Itn}^\top\vlt  + \vg_{\Itn})}_2 = \norm{X_{\Itn}^\top\vlt + \vg_{\Itn}}_2.
	\]
	This, upon applying the triangle inequality, gives us
	\[
	\norm{\btn_{\Itn} - \bo_{\Itn}}_2 \leq 2\norm{X_{\Itn}^\top\vlt + \vg_{\Itn}}_2.
	\]
	Now, using the SGSC and SGSS properties of $X$ (since $\text{G-supp}(\Itn) \leq k + \ko$), we can show that $\norm{X_{\Itn}^\top\vlt}_2 = \norm{X_{\Itn}^\top(XX^\top)^{-1}X^\top_{\It}(\bt - \bo)}_2 \leq \frac{\Lambda_{k+\ko}}{\lambda_{\frac{n}{d}}}\norm{\bt - \bo}_2$.
	
	Since $\veps$ is a Gaussian vector, using tail bounds for Chi-squared random variables (for example, see \cite[Lemma 20]{BhatiaJK2015}), for any set $S$ of size $(k + \ko)d$, we have with probability at least $1 - \delta$ ,$\norm{\veps_S}_2^2 \leq \sigma^2 (k + \ko)d + 2e\sigma^2\sqrt{6(k + \ko)d\log\frac{1}{\delta}}$. Taking a union bound over all sets of group size $(k + \ko)$ and $\binom{n/d}{k} \leq \br{\frac{en}{kd}}^k$ gives us, with probability at least $1 - \delta$, for all sets $S$ of group size at most $(k + \ko)$,
	\[
	\norm{\veps_S}_2 \leq \sigma\sqrt{(k + \ko)}\sqrt{d + 2e\sqrt{6 d \log\frac{en}{\delta(k + \ko)d}}}
	\]
	
	From Lemma~\ref{lem:Xe-bound}, with probability at least $1 - \delta$, we have $\norm{X \veps}_2 \leq 2 \sigma \sqrt{n \sqrt{\log{n}} c' d \log{\frac{2d}{\delta}}}$. This allows us to bound $\norm{\vg_{\Itn}}_2$
	\begin{align*}
		\norm{\vg_{\Itn}}_2 &= \norm{\veps_{\Itn} - X_{\Itn}^\top(XX^\top)^{-1}X\veps}_2\\
		&\leq \sigma\sqrt{(k+\ko)}\sqrt{d + 2e\sqrt{6 d \log\frac{en}{\delta(k+\ko)d}}} + 2 \sigma\frac{\sqrt{\Lambda_{k+\ko}}}{\lambda_{\frac{n}{d}}} \sqrt{n \sqrt{\log{n}} c' d \log{\frac{2d}{\delta}}}\\
		&\leq \underbrace{\sigma\sqrt{(k+\ko)}\sqrt{d + 2e\sqrt{6 d \log\frac{en}{\delta(k+\ko)d}}}}_{e_0} \br{1 + 2 f(\wo,\sigma) \sqrt{\frac{c' d \br{\log{n}}^{3/2}}{n}\log\frac{2d}{\delta}}}\\
		&= 1.0003 e_0,
	\end{align*}
	where the second last step is due to Remark~\ref{rough-bounds} for sufficiently large enough $n$ so that $\sqrt n \ll n$. Note that $e_0$ does note depend on the iterates and is thus, a constant. This gives us
	\[
	\norm{\btn - \bo}_2 \leq \frac{2\Lambda_{k+\ko}}{\lambda_{\frac{n}{d}}}\norm{\bt - \bo}_2 + 2.0006 e_0.
	\]
	For data matrices sampled from AO-AR(d) ensembles, whose SGSC and SGSS properties are established in Theorem~\ref{ao-ssc-sss}, assuming $n \geq d \log d$, we have $e_0 = \cO\br{\sigma\sqrt{(k+\ko)d\log\frac{n}{\delta(k+\ko)d}}}$. Thus, if $\frac{\Lambda_{k+\ko}}{\lambda_{\frac{n}{d}}} < \frac{1}{4}$ (which is guaranteed by Remark~\ref{rough-bounds} with probability at least $1-\delta$ for $\ko \leq k \leq C \frac{\mathfrak{m}_{\wo}}{\cM_{\wo} + \cM_{W}}\frac{n}{d \log{n}}$), then in $T_0 = \cO\br{\log\frac{\norm{\bo}_2}{e_0}} = \cO\br{\log\frac{\norm{\bo}_2}{\sqrt{n}}}$ steps, \aoard ensures that $\norm{\vb^{T_0} - \bo}_2 \leq 4.0015 e_0$.
	
\end{proof}

\subsection{Fine Convergence Analysis}

\begin{replemma}{lem:lambda-bound-ts}
	Suppose $\ko \leq k \leq n/(C'd\log n)$ for some large enough constant $C'$. Then with probability at least $1 - \delta$, \myalgo ensures at every time instant $t > T_0$
	\[
	\frac{C}{\lambda_n}(1+\frac{\Lambda_n}{\lambda_n})\norm{X_\fan(X_\fan^\top\vlt + \vg_\fan)}_2 \leq 0.5\norm{\vlt}_2 + \bigO{\sigma\sqrt{\frac{d\log n}{n}\log\frac{1}{\delta}}}
	\]
\end{replemma}
\begin{proof}
	As before, we change the problem so that instead of thresholding the top $k$ elements of the vector $X^\top\vlt + \vg$ by magnitude, we threshold all elements which exceed a certain value $\tau$ in magnitude. Again as before, we show that with high probability, for sufficiently small $k$, the $k\nth$ largest element of the vector will have a large magnitude.
	
	Proving the second part of the result is relatively simple in the time series setting because of the error tolerance bound $\ko \leq n/(d\log n)$ that we assume in this setting. For sake of simplicity, as well as without loss of generality, assume as before that $\sigma = 1$. Then using the tail bounds for martingales with sub-Gaussian entries from \cite{Shamir2011}, we can yet again show that with probability at least $1 - \exp\br{\Omega(n)}$, at least a 1/50 fraction of points in the vector $\vg$ will exceed the value 1.75 in magnitude. 
	
	Now, using the subset smoothness of the data matrix $X$ from Theorem~\ref{ao-ssc-sss} on subsets of size $1$ tells us that $\max_i\norm{X_i}_2 \leq \Lambda_1 \leq \bigO{d\log n}$, where $X_i$ is the $i\nth$ column of the data matrix $X$. Note that this also includes the influence of the error vector $\ve^\ast$. Thus, if we assume $\ko \leq k < \bigO{\frac{1}{d\log n}}$ then $\|\vla^{T_0}\|_2 \leq \frac{1}{4\max_i\norm{X_i}_2}$ which gives us $\norm{X^\top\vlt}_\infty \leq \frac{1}{4}$. This assures us that for any $k < n/50$, the $k$ largest elements by magnitude in the vector $X^\top\vlt + \vg$ will be larger than 1.5.
	
	Having assured ourselves of this, we move on to the analysis assuming that thresholding is done by value and not by cardinality. Let $\vz = [z_1,z_2,\ldots,z_n]$ where $z_i = (X_i^\top\vla + g_i)\cdot\ind{\abs{X_i^\top\vla + g_i} > \tau}$. We have
	\[
	X\vz = \sum_{j=1}^nX_jz_j = \sum_{i=1}^nX_j(X_j^\top\vla + g_j)\cdot\ind{\abs{X_j^\top\vla + g_j} > \tau},
	\]
	where the previous result ensures that we can set $\tau \geq 1.5$, as well as safely assume that $\abs{X_j^\top\vla} \leq 0.25$. In the following, we analyze the $i\nth$ coordinate of the vector i.e.
	\[
	(X\vz)_i = \sum_{i=1}^nX^i_j(X_j^\top\vla + g_j)\cdot\ind{\abs{X_j^\top\vla + g_j} > \tau} =: \sum_{i=1}^n\zeta_i.
	\]
	We notice that $g_i|X_i \sim \cN(0,\sigma^2)$ which allows us to construct the following martingale difference sequence
	\[
	\sum_{i=1}^n\zeta_i - \E{\zeta_i|g_1,g_2,\ldots,g_{i-1}}
	\]
	We also note that the elements of the above sequence are conditionally sub-Gaussian with the sub-Gaussian norm at most $\bigO{\log n}$. Then using the Azuma style inequality for martingales with sub-Gaussian tails from \cite{Shamir2011} gives us, with high probability
	\[
	\sum_{i=1}^n\zeta_i - \E{\zeta_i|g_1,g_2,\ldots,g_{i-1}} \leq \sqrt{\frac{112\log n}{n}\log\frac{1}{\delta}}
	\]
	Note that $\E{\zeta_i|g_1,g_2,\ldots,g_{i-1}} = X^i_j\cdot\E{(X_j^\top\vla + g_j)\cdot\ind{\abs{X_j^\top\vla + g_j} > \tau}|g_1,g_2,\ldots,g_{i-1}}$. Also note that $(X_j^\top\vla + g_j)$ is conditionally distributed as $\cN(X_j^\top\vla,1)$ as we have assumed $\sigma = 1$ for simplicity. For a Gaussian variable $Y \sim \cN(\mu,1)$, we have
	\[
	\E{Y\cdot\ind{|Y| > \tau}} = \mu - \E{Y\cdot\ind{|Y| \leq \tau}} = \frac{\phi(\tau-\mu) - \phi(-\tau-\mu)}{\Phi(-\tau - \mu) - \Phi(\tau - \mu)},
	\]
	where $\phi(\cdot)$ and $\Phi(\cdot)$ are respectively, the density and cumulative distribution functions of the standard normal variable. Now, applying the mean value theorem gives us
	\[
	\abs{{\phi(\tau-\mu) - \phi(-\tau-\mu)}} = \abs{{\phi(\tau+\mu) - \phi(\tau-\mu)}} = 2\abs{\eta\phi(\eta)\mu},
	\]
	for some $\eta \in [\tau - \mu, \tau + \mu]$. For the ensured values of $\tau = 1.25$ and $\abs{\mu} \leq 0.25$, we have $\abs{{\phi(\tau-\mu) - \phi(-\tau-\mu)}} < 0.25$. For the same values we have $\Phi(-\tau - \mu) - \Phi(\tau - \mu) \geq 0.68$. Putting these together, we get
	\[
	\abs{(X\vz)_i} \leq C\tau\cdot\abs{\sum_{j=1}^nX^i_jX_j^\top\vla} + D
	\]
	where $\abs{C_\tau} \leq 0.4$ and $D \leq \sqrt{\frac{112\log n}{n}\log\frac{1}{\delta}}$. We note that this value of $C_\tau$ can be made arbitrarily small by simply requiring that $\ko \leq k < \frac{n}{C'\cdot d\log n}$ for a large enough constant $C' > 0$. In particular, we set $k,\ko$ such that $C_\tau \leq 0.9\frac{\Lambda_n}{\lambda_n}$. This gives us
	\[
	\frac{1}{\lambda_n}\norm{X\vz}_2 \leq \frac{C_\tau}{\lambda_n}\norm{XX^\top\vla}_2 + \frac{d}{\lambda_n}D \leq 0.5\norm{\vla}_2 + \bigO{\sqrt{\frac{d\log n}{n}\log\frac{1}{\delta}}},
	\]
	which concludes the proof.
\end{proof}

\end{document}